\documentclass[11pt,english]{article}

\usepackage{times}
\usepackage[dvipdfmx]{graphicx}
\usepackage{subfigure} 

\usepackage[round]{natbib}

\usepackage{algorithm}
\usepackage{algorithmic}
\usepackage{amsmath,amssymb}
\usepackage{amsthm}
\usepackage{mathrsfs}

\newenvironment{proof*}{\noindent{\bf Proof:}}{}
\newcommand{\ignore}[1]{}

\newcommand{\EE}{\mathrm{E}}

\newcommand{\Real}{\mathbb{R}}

\newcommand{\calY}{\mathcal{Y}}
\newcommand{\calX}{\mathcal{X}}

\newcommand{\Eqref}[1]{Eq.~{\eqref{#1}}}

\newcommand{\Diag}{\mathrm{Diag}}
\newcommand{\DiagI}{\mathrm{Diag}_{\mathcal{I}}}

\newcommand{\wt}[1]{w^{(#1)}}
\newcommand{\wstar}{w^*}

\newcommand{\xt}[1]{x^{(#1)}}
\newcommand{\yt}[1]{y^{(#1)}}
\newcommand{\xhatt}[1]{\hat{x}^{(#1)}}
\newcommand{\yhatt}[1]{\hat{y}^{(#1)}}
\newcommand{\lambdat}[1]{w^{(#1)}}
\newcommand{\lambdastar}{w^*}
\newcommand{\lambdatilde}[1]{\tilde{w}^{(#1)}}
\newcommand{\xstar}{x^*}
\newcommand{\ystar}{y^*}
\newcommand{\ystarhat}{\widehat{y}^*}

\newcommand{\boldzero}{\boldsymbol{0}}

\newcommand{\Ker}{\mathrm{Ker}}

\newcommand{\nhat}{\hat{n}}

\newcommand{\prox}{\mathrm{prox}}
\newcommand{\sign}{\mathrm{sign}}

\newcommand{\eyemat}[1]{\mathrm{I}_{#1}}

\newtheorem{Theorem}{Theorem}

\newtheorem{Lemma}[Theorem]{Lemma}

\newtheorem{Assumption}{Assumption}

\newlength{\myfigwidth}
\title{Stochastic Dual Coordinate Ascent \\ with Alternating Direction Multiplier Method}

\author{Taiji Suzuki \\ Department of Mathematical and Computing Sciences, \\ Tokyo Institute of Technology, \\ Tokyo 152-8552, JAPAN \\
 {\tt s-taiji@is.titech.ac.jp}}

\date{}

\begin{document} 



\allowdisplaybreaks[1]

\maketitle

\begin{abstract} 
We propose a new stochastic dual coordinate ascent technique 
that can be applied to a wide range of regularized learning problems.
Our method is based on 
Alternating Direction Multiplier Method (ADMM) to deal with complex regularization functions
such as structured regularizations.
Although the original ADMM is a batch method, 
the proposed method offers a stochastic update rule where each iteration requires only one or few sample observations.
Moreover, our method can naturally afford mini-batch update and it gives speed up of convergence.
We show that, under mild assumptions, our method converges exponentially.
The numerical experiments show that our method actually performs efficiently.

\noindent {\bf Keywords}: Stochastic Dual Coordinate Ascent, Alternating Direction Multiplier Method, Exponential Convergence, Structured Sparsity.
\end{abstract} 

\section{Introduction}
This paper proposes a new stochastic optimization method 
that shows exponential convergence
and can be applied to wide range of regularization functions using the techniques of stochastic dual coordinate ascent with 
alternating direction multiplier method.
Recently, it is getting more and more important to develop an efficient optimization method which can handle large amount of samples.
One of the most successful approaches is a stochastic optimization approach.
Indeed, a lot of stochastic methods have been proposed to deal with large amount of samples.
Among them, the (online) stochastic gradient method is the most basic and successful one.
This can be naturally applied to the regularized learning frame-work.
Such a method is called several
different names including online proximal gradient descent, 
forward-backward splitting and online mirror descent~\citep{JMLR:Duchi+Singer:2009}.
Basically, these methods are intended to process 
sequentially coming data.
They update the parameter using one new observation and discard the observed sample.
Therefore, they don't need large memory space to store the whole observed data.
The convergence rate of those methods is $O(1/\sqrt{T})$ for general settings and $O(1/T)$ for strongly convex losses,
which are minimax optimal~\citep{Book:Nemirovskii+Yudin:1983}.

On the other hand, recently it was shown that, if it is allowed to reuse the observed data several times,
it is possible to develop a stochastic method with exponential convergence rate for a strongly convex objective~\citep{NIPS:LeRoux+Schmidt+Bach:2012,arXiv:Shai+Tong:2012,JMLR:SDCA:Shai+Tong:2013}.
These methods are still stochastic in a sense that one sample or small mini-batch is randomly picked up to be used for each update. 
The main difference from the stochastic gradient method is that these methods are intended to 
process data with a fixed number of training samples.
Stochastic Average Gradient (SAG) method~\citep{NIPS:LeRoux+Schmidt+Bach:2012}
utilizes an {\it averaged} gradient to show an exponential convergence. 
Stochastic Dual Coordinate Ascent (SDCA) method 
solves the dual problem using a stochastic coordinate ascent technique \citep{arXiv:Shai+Tong:2012,JMLR:SDCA:Shai+Tong:2013}.
These methods have favorable properties of both online-stochastic approach and batch approach.
That is, they show fast decrease of the objective function in the early stage of the optimization as online-stochastic approaches, 
and shows exponential convergence after the ``burn in'' time as batch approaches.  
However, these methods have some drawbacks.
SAG needs to maintain all gradients computed on each training sample in memory which amount to dimension times sample size.
SDCA method can be applied only to a simple regularization function for which the dual function is easily computed,
thus it is hard to apply the method to a complex regularization function such as structured regularization.

In this paper, we propose Stochastic Dual Coordinate Ascent method for Alternating Direction Multiplier Method (SDCA-ADMM).
Our method is similar to SDCA, but inherits a favorable property of ADMM.
By combining SDCA and ADMM, our method can be applied to a wide range of regularized learning problems.
ADMM is an effective optimization method to solve a composite optimization problem described as $\min_x f(x) + g(y)~\mathrm{s.t.}~Ax+By=0$~\citep{CMA:Gabay+Mercier:1976,FTML:Boyd+etal:2010,JMLR:Qin+Goldfarb:2012}.
This formulation is quite flexible and fit wide range of applications
such as structured regularization, dictionary learning, convex tensor decomposition and so on~\citep{JMLR:Qin+Goldfarb:2012,ICML:Jacob+etal:2009,NIPS:Tomioka+etal:2011,NeuroC:Rakotomamonjy:2013}.
However, ADMM is a batch optimization method.
Our approach transforms ADMM to a stochastic one by utilizing stochastic coordinate ascent technique. 
Our method, SDCA-ADMM, does not require large amount of memory because it observes only one or few samples for each iteration.
SDCA-ADMM can be naturally adapted to a sub-batch situation where a block of few samples is utilized for each iteration.
Moreover, it is shown that our method shows exponential convergence for risk functions with some strong convexity and smoothness property. 
The convergence rate is affected by the size of sub-batch. If the samples are not strongly correlated, sub-batch gives a better convergence rate than one-sample update.

\section{Structured Regularization and its Dual Formulation}
In this section, we give the problem formulation of structured regularization and its dual formulation.
The standard regularized risk minimization is described as follows:
\begin{align}
\min_{w \in \Real^p} \frac{1}{n} \sum_{i=1}^n f_i(z_i^\top w) + \tilde{\psi}(w),
\label{eq:SimpleOptimization}
\end{align} 
where $z_1,z_2,\dots,z_n $ are vectors in $\Real^p$,
$w$ is the weight vector that we want to learn, 
$f_i$ is a loss function for the $i$-th sample,
and $\tilde{\psi}$ is the regularization function which is used to avoid over-fitting.
For example, the loss function $f_i$ can be taken as 
a classification surrogate loss $f_i(z_i^\top w) = \ell(y_i,z_i^\top w)$ where 
$y_i$ is the training label of the $i$-th sample.
With regard to $\tilde{\psi}$, we are interested in a sparsity inducing regularization,
e.g., $\ell_1$-regularization, group lasso regularization, trace-norm regularization, and so on. 
Our motivation in this paper is to deal with a ``complex'' regularization $\tilde{\psi}$ where it is not easy to directly minimize the regularization function 
(more precisely the proximal operation determined by $\tilde{\psi}$ is not easily computed, see \Eqref{eq:prox_definition}).
This kind of regularization appears in, for example, structured sparsity such as overlapped group lasso and graph regularization~\citep{ICML:Jacob+etal:2009,TechRepo:Signoretto+etal:2010}.
In many cases, 
such a ``complex'' regularization function can be decomposed into a ``simple'' regularization $\psi$ and a linear transformation $B$,
that is, $\tilde{\psi}(w) = \psi(B^\top w)$ where $B \in \Real^{p\times d}$ . Using this formulation, the optimization problem (\Eqref{eq:SimpleOptimization}) is equivalent to 
\begin{align}
\min_{w \in \Real^p} \frac{1}{n} \sum_{i=1}^n f_i(z_i^\top w) + \psi(B^\top w).
\label{eq:CoreOptimizationProblem}
\end{align} 
The purpose of this paper is to give an efficient stochastic optimization method to solve this problem \eqref{eq:CoreOptimizationProblem}.
For this purpose, we employ the {\it dual formulation}.
Using the {\it Fenchel's duality theorem}, we have the following dual formulation.
\begin{Lemma}
\label{lemm:OptimalityAndDuality}
\begin{align}
&\min_{w \in \Real^p} \frac{1}{n} \sum_{i=1}^n f_i(z_i^\top w) + \psi(B^\top w) \notag\\
=  
&
- \!\! \min_{x \in \Real^n, y \in \Real^d}\left\{ \frac{1}{n}\sum_{i=1}^n f_i^*(x_i) + \psi^*\big(\frac{y}{n}\big) \mid Z x + By=0 \right\},
\label{eq:LemmaMainDuality}
\end{align}
where $f_i^*$ and $\psi^*$ are the convex conjugates of $f_i$ and $\psi$ respectively~\citep{Book:Rockafellar:ConvexAnalysis}\footnote{The convex conjugate function $f^*$ of 
$f$ is defined by $f^*(y) := \sup_x\{x^\top y - f(x)\}$.},
and $Z = [z_1,z_2,\dots,z_n] \in \Real^{p \times n}$. 
Moreover $w^*$, $x^*$ and $y^*$ are optimal solutions of both sides if and only if 
\begin{align*}
&z_i^\top w^* \in \nabla f_i^*(x_i^*),~
\frac{1}{n} y^* \in \nabla \psi(u) |_{u=B^\top w^*}, \\
&Z x^* + B y^* = 0.
\end{align*}

\end{Lemma}
\begin{proof}
By Fenchel's duality theorem (Corollary 31.2.1 of \citet{Book:Rockafellar:ConvexAnalysis}), we have that
\begin{align}
&\min_{w \in \Real^p} \frac{1}{n} \sum_{i=1}^n f_i(z_i^\top w) + \tilde{\psi}(w)  
=  
- \min_{x \in \Real^n} \left\{ \frac{1}{n}\sum_{i=1}^n f_i^*(x_i) + \tilde{\psi}^*(-Z x/n) \right\}.
\label{eq:PrimalDualInLemma}
\end{align}
Moreover $x^*,w^*$ are optimal of each side if and only if $z_i^\top w^* \in \nabla f_i^*(x^*_i)$ and $-Z x^*/n \in \nabla \tilde{\psi}(w^*) = B \nabla \psi(u^*)|_{u=B^\top w^*}$ 
(Corollary 31.3 of \citet{Book:Rockafellar:ConvexAnalysis}).
Now, Theorem 16.3 of \citet{Book:Rockafellar:ConvexAnalysis} gives that 
$$
\tilde{\psi}^*(u) = (\psi\circ B^\top)^*(u) = \inf\{\psi^*(y) \mid By = u\}.
$$
Thus $\tilde{\psi}^*(-Zx/n) = \inf\{\psi^*(y/n) \mid By = -Zx \}$, and substituting this into the RHS of \Eqref{eq:PrimalDualInLemma} 
we obtain \Eqref{eq:LemmaMainDuality}.
Now, $y^*$ satisfying $Zx^* + By^* = 0$ is the optimal value if and only if 
$\psi^*(y^*/n) = \tilde{\psi}^*(-Zx^*/n)$ for the optimal $x^*$.
Thus, if $(w^*,x^*,y^*)$ is optimal, then we have $-Z x^*/n \in \nabla \tilde{\psi}(w^*)$ and thus
$\psi^*(y^*/n) = \tilde{\psi}^*(-Zx^*/n) = \langle w^*, -Zx^*/n \rangle - \tilde{\psi}(w^*) 
= \langle B^\top w^*, y^*/n \rangle - \psi(B^\top w^*)$ which implies $y^*/n \in \nabla \psi(u)|_{u=B^\top w^*}$.
Contrary, if $y^*/n \in \nabla \psi(u)|_{u=B^\top w^*}$, then it is obvious that $-Z x^*/n \in \nabla \tilde{\psi}(w^*)$
because $Z x^* + B y^* = 0$.
Therefore, we obtain the optimality conditions.


\end{proof}
The dual problem is a composite objective function optimization with a linear constraint $Z x + By = 0$.
In the next section, we give an efficient stochastic method to solve this dual problem.
A nice property of the dual formulation is that, in many machine learning applications,
the dual loss function $f_i^*$ becomes strongly convex.
For example, for the logistic loss $f_i(x) = \log(1+\exp(-y_i x))$, the dual function is $f_i^*(-u) = y_i u\log(y_i u) + (1-y_i u)\log(1-y_i u)~(y_i u\in [0,1])$
and its modulus of strong convexity is much better than the primal one.
More importantly, each sample $(z_i,y_i)$ directly affects only each coordinate $x_i$ of dual variable.
In other words, if $x_i$ is fixed the $i$-th sample $(z_i,y_i)$ has no influence to the objective value.
This enables us to utilize the stochastic coordinate ascent technique in the dual problem because update of single coordinate $x_i$ requires only the information 
of the $i$-th sample $(z_i,y_i)$.

Finally, we give the precise notion of the ``complex'' and ``simple'' regularizations.
This notion is defined by the computational complexity of {\it proximal operation} corresponding to the regularization function~\citep{Book:Rockafellar:ConvexAnalysis}.
The proximal operation corresponding to a convex function $\psi$ is defined by 
\begin{align}
\label{eq:prox_definition}
\prox(q | \psi) := \mathop{\arg\min}_{u} \left\{ \frac{1}{2}\| q - u\|^2 + \psi(u)  \right\}.
\end{align}
For example, the proximal operation corresponding to $\ell_1$-regularization $\psi(w) = \|w\|_{\ell_1}$ is easily computed
as $\prox(q | \psi) = (\sign(w_i)\max\{|w_i| -1,0\})_{i}$ which is the so-called soft-thresholding operation.
More generally, the proximal operation for group lasso regularization with non-overlapped groups can also be analytically computed.
On the other hand, for overlapped group regularization, the proximal operation is no longer analytically obtained.
However, by choosing $B$ appropriately, we can split the overlap and obtain $\psi$ for which the proximal operation is easily computed
(see Section \ref{sec:NumericExp} for concrete examples).


\section{Proposed Method: Stochastic Dual Coordinate Ascent with ADMM}
\label{sec:ProposedMethod}
In this section, we present our proposal, stochastic dual coordinate ascent type ADMM. 
For a positive semidefinite matrix $S$, we denote by $\|x\|_S:=\sqrt{x^\top S x}$.
$Z_i$ denotes the $i$-th column of $Z$, which is $z_i$, and $Z_{\backslash i}$ is a matrix 
obtained by subtracting $i$-th column from $Z$.
Similarly, for a vector $x$, $x_{\backslash i}$ is a vector obtained by subtracting $i$-th component from $x$.

\subsection{One Sample Update of SDCA for ADMM}
The basic update rule of our proposed method in the $t$-th step is given as follows:
Each update step, choose $i\in \{1,\dots,n\}$ uniformly at random, and update as  
\begin{subequations}
\begin{align}
\yt{t}  \!\! \leftarrow 
&\mathop{\arg \min}_{y}\Big\{ n \psi^*(y/n) -\langle \wt{t-1}\!\!, Z\xt{t-1}\! + \! By \rangle  + \frac{\rho}{2}\|Z\xt{t-1} + By\|^2
+\frac{1}{2}\|y - \yt{t-1}\|_{Q}^2\Big\}, \label{eq:TheOriginalSingleVarUpdateYt} 
\\
\xt{t}_i \!\! \leftarrow 
&\mathop{\arg \min}_{x_i} \Big\{  f_i^*(x_i) - \langle \wt{t-1}, Z_ix_i + B\yt{t}\rangle + \frac{\rho}{2}\|Z_ix_i + Z_{\backslash i} \xt{t-1}_{\backslash i} + B\yt{t}\|^2 \notag \\
&~~~~~~~~~~~~~~~~+\frac{1}{2}\|x_i - \xt{t-1}_i\|_{G_{ii}}^2 \Big\}, \label{eq:TheOriginalSingleVarUpdateXt} 
\\
\wt{t} \!\! \leftarrow 
&\wt{t-1} - \gamma\rho\{n(Z\xt{t} + B\yt{t}) - (n-1)(Z\xt{t-1} + B\yt{t-1})\},
\end{align}
\end{subequations}
where $\wt{t}\in \Real^p$ is the primal variable at the $t$-th step,
$Q$ and $G$ are arbitrary positive semidefinite matrices, and $\gamma,\rho>0$ are parameters we give beforehand.

The optimization procedure looks a bit complicated, 
To simplify the procedure, we set $Q$ as 
\begin{align}
\label{eq:QetaB}
Q = \rho(\eta_B \eyemat{d} -  B^\top B)
\end{align}
where $\eta_B$ are chosen so that $\eta_B \eyemat{d} \succ B^\top B$.
Then, by carrying out simple calculations and denoting $\eta_{Z,i} = G_{ii}/\rho + \|z_i\|^2$, the update rule of $\xt{t}$ and $\yt{t}$ is rewritten as 
\begin{subequations}
\label{eq:proxupdateYtXt}
\begin{align}
\yt{t} \!\!   
\leftarrow  
\prox &\Big( 
\yt{t-1} + 
\frac{B^\top}{\rho \eta_B} \{\wt{t-1} -  \rho (Z \xt{t-1} + B \yt{t-1})\} 
~\Big|~ \frac{n\psi^*(\cdot/n)}{\rho \eta_B} \Big),
\label{eq:proxupdateYt}
 \\
\xt{t}_i   
\leftarrow \prox & \Big( 
\xt{t-1}_i +  \frac{Z_i^\top}{\rho \eta_{Z,i}} \{\wt{t-1} - \rho (Z \xt{t-1} + B \yt{t})\}~\Big|~ \frac{f_i^*}{\rho \eta_{Z,i}}\Big).
\label{eq:proxupdateXt}
\end{align}
\end{subequations}
Note that the update \eqref{eq:proxupdateXt} of $\xt{t}$ is just a one dimensional optimization, thus it is quite easily computed.
Moreover, for some loss functions such as the smoothed hinge loss used in Section \ref{sec:NumericExp}, we have an analytic form of the update.

The update rule \eqref{eq:proxupdateYt} of $\yt{t}$ can be rewritten by the proximal operation corresponding to the primal function $\psi$
while the rule \eqref{eq:proxupdateYt} is given by that corresponding to the dual function $\psi^*$.
Indeed, there is a clear relation between primal and dual (Theorem 31.5 of \citet{Book:Rockafellar:ConvexAnalysis}): 
$$
\prox(q | \psi) + \prox(q|\psi^*) = q.
$$
Using this, for $q^{(t)} = \yt{t-1} + \frac{B^\top}{\rho \eta_B} \{\wt{t-1} - \rho (Z \xt{t-1} + B \yt{t-1})\}$, we have that 
\begin{align}
\label{eq:ytUpdateSimple}
\yt{t} \leftarrow q^{(t)} - \prox(q^{(t)} | n \psi(\rho \eta_B~\cdot~)/(\rho \eta_B)),
\end{align}
because $(cf(\cdot))^*(y) = cf^*(y/c)$ for a convex function $f$ and $c > 0$.
This is efficiently computed because we assumed the proximal operation corresponding to $\psi$ can be efficiently computed.

During the update, we need $Z\xt{t-1}$ which seems to require $O(n)$ computation at the first glance. 
However, it can be incrementally updated as $Z\xt{t} = Z\xt{t-1} + Z_i(\xt{t}_i - \xt{t-1}_i)$.
Thus we don't need to road all the samples to compute $Z\xt{t-1}$ at each iteration.

In the above, the update rule of our algorithm is based on one sample observation.
Next, we give a mini-batch extension of the algorithm where more than one samples could be used for each iteration.

\subsection{Mini-Batch Extension}
\label{sec:implementationtech}
Here, we generalize our method to mini-batch situation
where, at each iteration, we observe a small number of samples $\{(x_{i_1},y_{i_1}),\dots,(x_{i_k},y_{i_k})\}$
instead of one sample observation.
At each iteration, we randomly choose an index set $I \subseteq \{1,\dots,n\}$
so that each index $i$ is included in $I$ with probability $1/K$;
$P(i \in I) = 1/K$  for all $i =1,\dots,n$.
To do so, we suggest the following procedure. 
We split the index set $\{1,\dots,n\}$ into $K$ groups $(I_1,I_2,\dots,I_K)$ beforehand,
and then pick up uniformly $k \in \{1,\dots,K\}$ and set $I = I_k$ for each iteration.
Each sub-batch $I_k$ can have different cardinality from others, but 
the probability $P(i \in I)$ should be uniform for all $i = 1,\dots,n$.
The update rule using sub-batch is given as follows:
Update $\yt{t}$ as before \eqref{eq:TheOriginalSingleVarUpdateYt}, 
and update $\xt{t}$ and $\wt{t}$ by
\begin{subequations}
\label{eq:MiniBatchUpdate}
\begin{align}
&\xt{t}_I  \leftarrow  \mathop{\arg \min}_{x_I}  \Big\{ \sum_{i \in I} f_i^*(x_i) - \langle \wt{t-1}, Z_I x_I + B\yt{t}\rangle 
 + \frac{\rho}{2}\|Z_Ix_I + Z_{\backslash I} \xt{t-1}_{\backslash I} + B\yt{t}\|^2  \notag \\
&~~~~~~~~~~~~~~~~~~~~~~~~~~~~  +\frac{1}{2}\|x_I - \xt{t-1}_I\|_{G_{I,I}}^2\Big\}, \label{eq:TheOriginalMinibatchUpdateXt} \\
&\wt{t}  \leftarrow  \wt{t-1} -  \gamma\rho\{n(Z\xt{t} + B\yt{t}) - (n-n/K)(Z\xt{t-1} + B\yt{t-1})\}.
\end{align}
\end{subequations}
Using $Q$ given in \Eqref{eq:QetaB}, the update rule of $\yt{t}$ can be replaced by \Eqref{eq:ytUpdateSimple} as in one-sample update situation.
The update rule of $\xt{t}$ can also be simplified by choosing $G$ appropriately.
Because sub-batches have no overlap between each other,
we can construct a positive semi-definite matrix $G$ such that 
the block-diagonal element $G_{I,I}$ has the form 
\begin{equation}
\label{eq:GIIetaZI}
G_{I,I}= \rho (\eta_{Z,I} - Z_I^\top Z_I)
\end{equation}
where $\eta_{Z,I}$ is a positive real satisfying 
$\eta_{Z,I} \geq \|Z_I^\top Z_I\|$.
The reason why we split the index sets into $K$ sets is to construct this kind of $G$ which ``diagonalizes'' the quadratic function in \eqref{eq:TheOriginalMinibatchUpdateXt}.
The choice of $I$ and $G$ could be replaced with another one for which we could compute the update efficiently,
as long as $P(i \in I)$ is uniform for all $i=1,\dots,n$.
Using $G$ given in \eqref{eq:GIIetaZI}, the update rule \eqref{eq:TheOriginalMinibatchUpdateXt} of $\xt{t}$ is rewritten as 
\begin{align}
\xt{t}_I \leftarrow 
\prox & \Big( 
\xt{t-1}_I + \frac{Z_I^\top}{\rho \eta_{Z,I}} \{\wt{t-1}  -\rho (Z \xt{t-1} 
 + B \yt{t})\} ~\Big|~ \frac{\sum_{i\in I} f_i^*}{\rho \eta_{Z,I}}\Big), 
\label{eq:MiniBatchUpdateXt}
\end{align}
where $x_I$ is a vector consisting of components with indexes $i\in I$, $x_I = (x_i)_{i\in I}$, and 
$Z_I$ is a sub-matrix of $Z$ consisting of columns with indexes $i\in I$, $Z_I = [Z_{i_1},\dots,Z_{i_{|I|}}]$.
Note that, since $\sum_{i\in I} f_i^*(x_i)$ is sum of single variable convex functions $f_i^*(x_i)$, 
the proximal operation in \Eqref{eq:MiniBatchUpdateXt} can be split into 
the proximal operation with respect to each single variable $x_i$.
This is advantageous for not only the simpleness of the computation but also parallel computation. 
That is, for $p_I = \xt{t-1}_I + \frac{Z_I^\top}{\rho \eta_{Z,I}} \{\wt{t-1} - \rho (Z \xt{t-1} + B \yt{t})\}$,
the update rule \eqref{eq:MiniBatchUpdateXt} is reduced to
$
\xt{t}_i \leftarrow \prox(p_i | \frac{f_i^*}{\rho \eta_{Z,I}})
$
for each $i \in I$, which is easily parallelizable.
In summary, our proposed algorithm is given in Algorithm \ref{alg:SDCAADMM}.
\begin{algorithm}[h]
   \caption{SDCA-ADMM}
   \label{alg:SDCAADMM}
\begin{algorithmic}
   \STATE {\bfseries Input: $\rho,\eta > 0$} 
   \STATE Initialize $x_0 = \boldzero$, $y_0 = \boldzero$, $w_0 = \boldzero$ and $\{I_1,\dots,I_K\}$.
   \FOR{$t=1$ {\bfseries to} $T$ } 
   \STATE Choose $k \in \{1,\dots,K\}$ uniformly at random, set $I = I_k$, and observe the training samples $\{(x_i,y_i)\}_{i \in I}$.
   \STATE Set $q^{(t)} \!\!=\! \yt{t-1} \!+ \frac{B^\top}{\rho \eta_B} \{\wt{t-1} \!-\! \rho (Z \xt{t-1} \!+\! B \yt{t-1})\}$.
   \STATE Update $\yt{t} \leftarrow q^{(t)} - \prox(q^{(t)} | n \psi(\rho \eta_B~\cdot~)/(\rho \eta_B))$
   \STATE Update $\xt{t}_I \leftarrow \prox \Big( \xt{t-1}_I + \frac{Z_I^\top}{\rho \eta_{Z,I}} \{\wt{t-1} - \rho (Z \xt{t-1} + B \yt{t})\} ~\Big|~ \frac{\sum_{i \in I}f_i^*}{\rho \eta_{Z,I}}\Big).$
   \STATE Update $\wt{t} \leftarrow \wt{t-1} - \gamma\rho\{n(Z\xt{t} + B\yt{t})- (n-n/K)(Z\xt{t-1} + B\yt{t-1})\}$.
   \ENDFOR
   \STATE {\bfseries Output: $\wt{T}$}.
\end{algorithmic}
\end{algorithm}

Finally, we would like to highlight the connection between our method and the original batch ADMM~\citep{JOTA:Hestenes:1969,Opt:Powell:1969,Roc76}.
The batch ADMM utilizes the following update rule
\begin{subequations}
\label{eq:BatchUpdate}
\begin{align}
\yt{t} \leftarrow &\mathop{\arg \min}_{y}\Big\{ n \psi^*\left(\frac{y}{n}\right) -\langle \wt{t-1}, Z\xt{t-1}  +  B y \rangle  
+ \frac{\rho}{2}\|Z\xt{t-1} + By\|^2\Big\}, \\
\xt{t} \leftarrow &\mathop{\arg \min}_{x}\big\{ {\textstyle \sum_{i=1}^n f_i^*(x_i)} - \langle \wt{t-1}, Zx + B\yt{t}\rangle 
+ \frac{\rho}{2}\|Zx + B\yt{t}\|^2 \big\}, \\ 
\wt{t} \leftarrow &\wt{t-1} - \gamma\rho(Z\xt{t} + B\yt{t}).
\end{align}
\end{subequations}
One can see that the update rule of our algorithm is reduced to 
that of the batch ADMM \eqref{eq:BatchUpdate} if we set $K=1$ except the term related to $G$ and $Q$ (the terms $\frac{1}{2}\|\cdot\|_Q^2$ and $\frac{1}{2}\|\cdot\|_{G_{I,I}}^2$).
These terms related to $G$ and $Q$ 
are used also in batch situation to eliminate cross terms in $BB^\top$ and $ZZ^\top$. 
This technique is called linearization.
The linearization technique makes the update rule simple and parallelizable, and in some situations 
makes it possible to obtain an analytic form of the update.

\section{Linear Convergence of SDCA-ADMM}
In this section, the convergence rate of our proposed algorithm is given.
Indeed, the convergence rate is exponential (R-linear).
To show the convergence rate, we assume some conditions.
First, we assume that there exits an unique optimal solution $w^*$ and $B^\top$ is injective (on the other hand, $B$ is not necessarily injective).
Moreover, we assume the uniqueness of the dual solution $\xstar$,
but don't assume the uniqueness of $\ystar$.
We denote by the set of dual optimum of $y$ as $\calY^*$ and assume that $\calY^*$ is compact.
Then, by Lemma \ref{lemm:OptimalityAndDuality}, we have that
\begin{align}
\label{eq:optimality_criterion}
z_i^\top \wstar \in \nabla f_i^*(\xstar_i),~~\ystar/n \in \nabla \psi(u)|_{u=B^\top \wstar}.
\end{align}
By the convex duality arguments, this implies that $\xstar_i \in \nabla f_i(u)|_{u=z_i^\top \wstar},~~B^\top \wstar \in \nabla \psi^*(u)|_{u=\ystar/n}$.

Moreover, we suppose that each (dual) loss function $f_i$ is locally $v$-strongly convex and $\psi$, $h$-smooth around the optimal solution 
and $\psi^*$ is also locally  strongly convex in a weak sense as follows.
\begin{Assumption}
\label{ass:LocalStrongConvexity}
There exits $v >0$ such that, $\forall x_i \in \Real$,
\begin{align*}
&f_i^*(x_i) - f_i^*(x_i^*) \geq \langle \nabla f_i^*(x_i^*), x_i - x_i^* \rangle + \frac{v\|x_i-x_i^*\|^2}{2}. 
\end{align*}
There exit $h>0$ and $v_\psi>0$ such that, for all $y$, there exists $\widehat{y}^* \in \calY^*$ such that
\begin{align}
\psi^*(y/n) - \psi^*(\widehat{y}^*/n) 
&\geq \langle B^\top \wstar, y/n - \widehat{y}^*/n \rangle  + \frac{v_\psi}{2}\| P_{\mathrm{Ker}(B)} (y/n - \widehat{y}^*/n)\|^2, \label{eq:DualStrongConvexityPsi} 
\end{align}
and for all $\ystar \in \calY^*$ we have 
\begin{align}
\psi(u) - \psi(B^\top \wstar) 
& \geq \langle \ystar/n,  u -  B^\top \wstar \rangle  + \frac{h}{2}\|u - B^\top \wstar\|^2,
\label{eq:PrimalStrongConvexityPsi}
\end{align}
where $P_{\mathrm{Ker}(B)}$ is the projection matrix to the kernel of $B$.
\end{Assumption}
Note that these conditions should be satisfied only around the optimal solutions $(\xstar,\ystar)$ and $\wstar$.
It does not need to hold for every point, thus is much weaker than the ordinal strong convexity.
Moreover, the inequalities need to be satisfied only for the solution sequence $(\wt{t},\xt{t},\yt{t})$ of our algorithm.
The condition \eqref{eq:DualStrongConvexityPsi} is satisfied, for example, by $\ell_1$-regularization 
because the dual of $\ell_1$-regularization is an indicator function with a compact support and, outside the optimal solution set $\calY^*$, the indicator function is lower bounded by a quadratic function.
In addition, the quadratic term in the right hand side of this condition \eqref{eq:DualStrongConvexityPsi} is 
restricted on $\mathrm{Ker}(B)$. This makes it possible to include several types of regularization functions.
Indeed, if $B=\eyemat{p}$, this condition is always satisfied.
The assumption \eqref{eq:PrimalStrongConvexityPsi} is the strongest assumption.
This is satisfied for elastic-net regularization. 
If one wants to obtain a solution for non-strongly convex regularization such as $\ell_1$-regularization,
just adding a small square term, we obtain an approximated solution which is sufficiently close to the true one within a precision.

Define the primal and dual objectives as 
\begin{align*}
\textstyle  F_P(w) := &\frac{1}{n} \sum_{i=1}^n f_i(z_i^\top w) + \psi(B^\top w), \\
\textstyle  F_D(x,y) := &\frac{1}{n} \sum_{i=1}^n f_i^*(x_i) + \psi^*(\textstyle\frac{y}{n}) 
- \langle \wstar, Z \frac{x}{n} - B\frac{y}{n} \rangle.
\end{align*}
Note that, by \Eqref{eq:optimality_criterion}, $F_P(w)-F_P(w^*)$ and $F_D(x,y)-F_D(x^*,y^*)$ are always non-negative.
Define the block diagonal matrix $H$ as $H_{I,I} = \rho Z_I^\top Z_I + G_{I,I}$ for all $I \in \{I_1,\dots,I_K\}$ and $H_{i,j} = 0$ for $(i,j) \notin I_k \times I_k~(\forall k)$.
Let $\|y - \calY^*\|_Q := \min\{\|y - \ystar\|_Q \mid \ystar \in \calY^*\}$. 
We define $R_D(x,y,w)$ as 
\begin{align*}
R_D(x,y,w) :=  
&F_D(x,y)- F_D(\xstar,\ystar) + \frac{1}{2 n \gamma \rho}\|w - \wstar\|^2 \notag \\   
& + \frac{\rho(1-\gamma)}{2n} \|Z x+B y\|^2 + 
\frac{1}{2n} \|x -\xstar\|^2_{v \eyemat{p} +  H} + \frac{1}{2n^2}\|y - \calY^*\|_Q^2.
\end{align*}
For a symmetric matrix $S$, we define $\sigma_{\mathrm{max}}(S)$ and $\sigma_{\mathrm{min}}(S)$ as 
the maximum and minimum singular value respectively.
\begin{Theorem}
\label{th:LinearConvergence} 
Suppose that $\gamma=\frac{1}{4n}$, $\eta_{Z,I}  > (1+2\gamma n(1-1/K))\sigma_{\max}(Z_I^\top Z_I)$ for all $I \in \{I_1,\dots,I_K\}$ and $B^\top$ is injective.
Then, under Assumption \ref{ass:LocalStrongConvexity}, the dual objective function converges R-linearly:
We have that, for $C_1 = R_D(\xt{0},\yt{0},\wt{0})$, 
$$
\EE[R_D(\xt{T},\yt{T},\wt{T})] \leq \left(1-\frac{\mu}{K}\right)^T C_1,
$$
where 
\begin{align*}
\mu = \min\bigg\{ \frac{v}{4(v+\sigma_{\mathrm{max}}(H))},
\frac{h\rho\sigma_{\min}(B^\top B)}{2\max\{1/n,4h\rho,4  h\sigma_{\max(Q)}\}},\frac{v_\psi}{4\sigma_{\max}(Q)}, 
\frac{n v\sigma_{\min}(BB^\top)}{4\sigma_{\max(Q)} (\rho \sigma_{\max}(Z^\top Z) + 4v)}  \bigg\}.
\end{align*}
In particular, we have that  
$$
\EE[ \|\wt{T} - \wstar\|^2]  \leq \frac{n \rho}{2}  \left(1-\frac{\mu}{K}\right)^T C_1.
$$
If we further assume $\psi(B^\top w) \leq \psi(B^\top w^*) + \langle y^*/n, B^\top(w- w^*)\rangle + l_1\|w- w^*\| + l_2\|w- w^*\|^2$~($\forall w$), 
then this implies that 
\begin{align*}
& \EE[ F_P(\wt{T}) - F_P(\wstar) ] \\
&\leq  \left( \frac{\sigma_{\max}(Z^\top Z/n)}{2 v} +  l_2 \right) 
\frac{n \rho}{2}  \left(1-\frac{\mu}{K}\right)^T C_1 +l_1 \sqrt{\frac{n \rho}{2}  \left(1-\frac{\mu}{K}\right)^T C_1}.
\end{align*}

\end{Theorem}
The proof is deferred to the appendix.
This theorem shows that the primal and dual objective values converge R-linearly. 
Moreover, the primal variable $w$ also converges R-linearly to the optimal value.
The number $K$ of sub-batches controls the convergence rate.  
If all samples are nearly orthogonal to each other, $\sigma_{\max}(H)$ is bounded by a constant for all $K$,
and thus convergence rate gets faster and faster as $K$ decreases (the size of each sub-batch grows up).
On the other hand, if samples are strongly correlated to each other, 
$\sigma_{\max}(H)$ grows linearly against $1/K$ and then the convergence rate is not improved by decreasing $K$.
As for batch settings, the linear convergence of batch ADMM has been shown by~\citep{TechRepo:Deng+Yin:2012}.
However, their proof can not be directly applied to our stochastic setting.
Our proof requires a technique specialized to stochastic coordinate ascent technique.
We would like to point out that the exponential convergence is not guaranteed 
if the choice of index set $I$ at each update is cyclic.
Thus the index $I$ should be randomly chosen. 
This is reported also in the paper~\citep{JMLR:SDCA:Shai+Tong:2013}.

The statement can be described in terms of the number of iterations required to achieve a precision $\epsilon$, i.e. smallest $T$ satisfying $\EE[ F_P(\wt{T}) - F_P(\wstar) ] \leq \epsilon$:
\begin{align*}
T \leq C' K 
\max\Bigg\{ 
&
\frac{v+\sigma_{\mathrm{max}}(H)}{v},
\frac{\max\{1/(nh), \rho,\sigma_{\max(Q)}\}}{\rho\sigma_{\min}(B^\top B)}, \\ 
& 
\frac{\sigma_{\max}(Q)}{ v_\psi},\frac{\sigma_{\max}(Q)(\rho \sigma_{\max}(Z^\top Z) + 4v)}{nv\sigma_{\min}(BB^\top)}  \Bigg\}\log\left(\frac{nC''}{\epsilon}\right), 
\end{align*}
where $C'$ and $C''$ are an absolute constant.
This says that dependency of $\epsilon$ is log-order.
An interesting point is that the influence of $h$, the modulus of local strong convexity of $\psi$. 
Usually the regularization function is made weaker as the number of samples increases.
In that situation, $h$ decreases as $n$ goes up.
However, even if we set $h=1/n$, we still have $T = O(K\log(n/\epsilon))$ instead of $O(nK\log(n/\epsilon))$.
Thus, the convergence rate is hardly affected by the setting of $h$.
This point is same as the ordinary SDCA algorithm \citep{JMLR:SDCA:Shai+Tong:2013}.
 
\section{Related Works}
In this section, we present some related works and discuss differences from our method.

The most related work is a recent study by \citep{arXiv:Shai+Tong:2012}
in which Stochastic Dual Coordinate Ascent (SDCA) method for a regularized risk minimization is proposed.
Their method also deals with the dual problem \eqref{eq:LemmaMainDuality} with $B = \eyemat{p}$ in our setting,
and apply a stochastic coordinate ascent technique. 
This method converges linearly.
At each iteration, the method solves the following one-dimensional optimization problem,
\begin{align*}
\Delta \xt{t}_i \! \leftarrow \mathop{\arg \min}_{\Delta x_i \in \Real} &~f_i^*(\Delta x_i + \xt{t-1}_i) + z_i^\top \wt{t-1} \Delta x_i  + \frac{1}{2 n} \|z_i\Delta x_i\|^2,
\end{align*}
and updates $\xt{t}_i \leftarrow \Delta \xt{t}_i + \xt{t-1}_i$ and $\wt{t} \leftarrow \nabla \tilde{\psi}^*(- A \xt{t})$. 
The most important difference from our method is the computation of $\nabla \psi^*$.
In a ``simple'' regularization function, it is often easy to compute the (sub-)gradient of $\tilde{\psi}^*$.
However, in a ``complex'' regularization such as structured regularization, the computation is not efficiently carried out. 
To overcome this difficulty, our method utilizes a linear transformed one $\psi(B\cdot) = \tilde{\psi}(\cdot)$,
and split the optimization with respect to $f_i^*$ and $\psi^*$ by applying ADMM technique. 
Thus, our method is applicable to much more general regularization functions.
Recently, a mini-batch extension of SDCA is a hot topic \citep{ICML:Takac+etal:2013,NIPS:Shwartz+Zhang:2013}.
Our approach realizes the mini-batch extension using the linearlization technique in ADMM which is naturally derived in the frame-work of ADMM.
Although the proof technique is quite different, 
the convergence analysis of normal mini-batch SDCA given by \citep{NIPS:Shwartz+Zhang:2013} 
is parallel to our theorem.

The second method related to ours is Stochastic Average Gradient (SAG) method \citep{NIPS:LeRoux+Schmidt+Bach:2012}.
The method is a modification of stochastic gradient descent method, but utilizes an {\it averaged} gradient.
A good point of their method is that we only need to deal with the primal problem.
Thus the computation is easy, and we don't need to look at the convex conjugate function.  
Moreover, their method also converges linearly. 
However, the biggest drawback is that, to compute the averaged gradient,
all gradients of loss functions computed at each sample should be stored in memory.
The memory size is usually $O(p\times n)$ which is hard to be stored for big data situation.
On the other hand, our method is free from such a memory problem.
Indeed, our method requires only $O(p + n)$-size memory.

The third method is online version of ADMM.
Recently some online variants of ADMM have been proposed by \citep{ICML:Wang+Banerjee:2012,ICML:Suzuki:2013,ICML:Ouyang+etal:2013}.
These methods are effective for complex regularizations as discussed in this paper.
Thus they are applicable to wide range of situations.
However, those methods are basically online methods, thus 
they discard the samples once observed.
They are not adapted to a situation where the training samples are observed several times. 
Therefore, the convergence rate is $O(1/\sqrt{T})$ in general and $O(\log(T)/T)$ for a strongly convex loss (possibly $O(1/T)$ with some modification).
On the other hand, our method converges linearly.

\section{Numerical Experiments}
\label{sec:NumericExp}
In this section, we give numerical experiments on artificial and real data to demonstrate 
the effectiveness of our proposed algorithm\footnote{All the experiments were carried out on Intel Core i7 2.93GHz with 8GB RAM.}.
We compare our SDCA-ADMM with the existing stochastic optimization methods 
such as Regularized Dual Averaging (RDA) \citep{JMLR:Duchi+Singer:2009,NIPS:Xiao:2009}, Online ADMM (OL-ADMM) \citep{ICML:Wang+Banerjee:2012},
Online Proximal Gradient descent ADMM (OPG-ADMM) \citep{ICML:Ouyang+etal:2013,ICML:Suzuki:2013}
and RDA-ADMM \citep{ICML:Suzuki:2013}.
We also compared our method with batch ADMM (Batch-ADMM) in the artificial data sets.
We used sub-batch with size 50 for all the methods including ours ($|I_k|=50~(\forall k)$, but $|I_K|$ could be less than 50).
We employed the parameter settings $\gamma=1/n$ and $\rho = 0.1$. 
As for $\eta_{Z,I}$ and $\eta_B$, we used $\eta_{Z,I} = 1.1\sigma_{\max}(Z_I^\top Z_I)$ and $\eta_B = \sigma_{\max}(B B^\top) + 1$.
All of the experiments are classification problems with structured sparsity.
We employed the {\it smoothed hinge loss}: 
$$
\textstyle
f_i(u) = \begin{cases} 0, &(y_i u \geq 1), \\ \frac{1}{2} - y_i u, &(y_i u < 0), \\ \frac{1}{2}(1-y_i u)^2, &(\text{otherwise}). \end{cases}
$$
Then the proximal operation with respect to the dual function of the smoothed hinge loss is analytically given by 
\begin{align*}
\textstyle
\prox(u|f_i^*/C) = \begin{cases} \frac{Cu-y_i}{1+C} & (-1 \leq \frac{C u y_i -1}{1+C} \leq 0), \\ -y_i & (-1 > \frac{Cu y_i -1}{1+C}), \\ 0 &(\text{otherwise}). \end{cases}
\end{align*}

\subsection{Artificial Data}
Here we execute numerical experiments on artificial data sets.
The problem is a classification problem with overlapped group regularization as performed in \citep{ICML:Suzuki:2013}.
We generated $n$ input feature vectors $\{z_i\}_{i=1}^n$ with dimension $d = 32 \times 32 = 1024$
where each feature is generated from i.i.d. standard normal distribution.
Then the true weight vector $w_0$ is generated as follows: First we generate a random matrix which has non-zero elements on its first column (distributed from i.i.d. standard normal) 
and zeros on other columns, and vectorize the matrix to obtain $w_0$. 
The training label $y_i$ is given by $y_i = \sign(z_i^\top w_0+ \epsilon_i)$ where $\epsilon_i$ is distributed from normal distribution with mean 0 and standard deviation $0.1$.

The group regularization is given as 
$
\tilde{\psi}(x) = C (\sum_{i=1}^{32} \|X_{:,i}\| + \sum_{j=1}^{32} \|X_{j,:}\| + 0.01 \times \sum_{i,j}X_{i,j}^2/2)
$
where $X$ is the $32\times 32$ matrix obtained by reshaping $x$.
The quadratic term is added to make the regularization function strongly convex\footnote{Even if there is no quadratic term, our method converged with almost the same speed.}.
Since there exist overlaps between groups, the proximal operation can not be straightforwardly computed~\citep{ICML:Jacob+etal:2009}. 
To deal with this regularization function in our frame-work,
we let $B^\top x = [x; x] (=[x^\top x^\top]^\top)$ and $\psi([x; x']) = C (\sum_{i=1}^{32} \|X_{:,i}\| + \sum_{j=1}^{32} \|X_{j,:}'\|)$.
Then we can see that $\tilde{\psi}(x) = \psi(B^\top x)$ and 
the proximal operation with respect to $\psi$ is analytically obtained; indeed it is easily checked that 
$\prox([q;q'] | \psi) = [\mathrm{ST}_{C'}(Q_{:,1}/(1+0.01C));\dots;\mathrm{ST}_{C'}(Q_{:,32}/(1+0.01C));\mathrm{ST}_{C'}(Q'_{1,:}/(1+0.01C));\dots;\mathrm{ST}_{C'}(Q'_{32,:}/(1+0.01C))]$ 
where $\mathrm{ST}_C(q) = q\max(1 - C/\|q\|,0)$ and $C'=C/(1+0.01C)$.

The original RDA requires a direct computation of the proximal operation for the overlapped group penalty.
To compute that, we employed the dual formulation proposed by \citep{NIPS:Yuan+etal:FOGLASSO:2011}.

We independently repeated the experiments 10 times and 
averaged the excess empirical risk ($F_P(\wt{t}) - \min_w F_P(w)$), the expected loss on the test data ($\EE_{(z,y)}[f(y,z^\top \wt{t})]$)
and the classification error ($\EE_{(z,y)}[1\{y \neq \sign(z^\top \wt{t})\}$).
Figure~\ref{fig:ArtificialData} shows these three values against CPU time with the standard deviation for $n=512$ and $n=5120$.
We employed $C_1=0.1/\sqrt{n}$.


\begin{figure}
\begin{center}
\subfigure[$n=512$]{
\includegraphics[width=\myfigwidth]{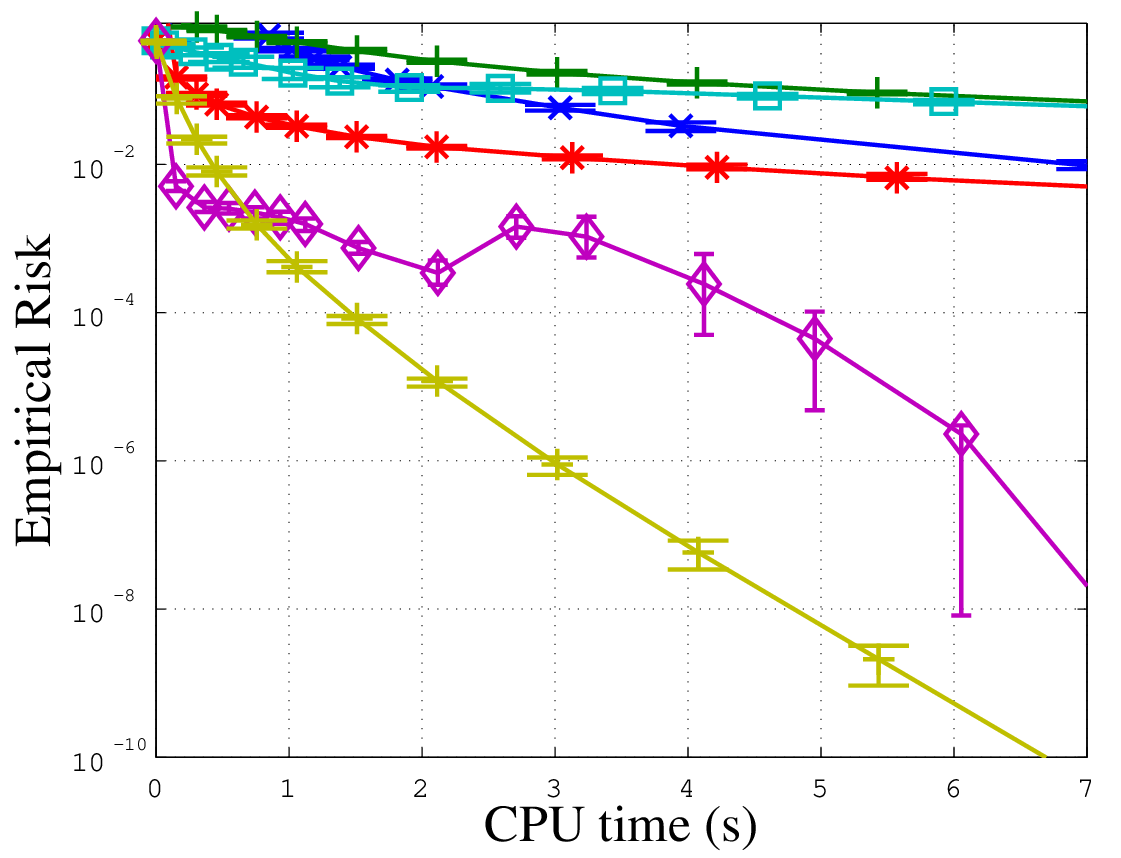} 
\includegraphics[width=\myfigwidth]{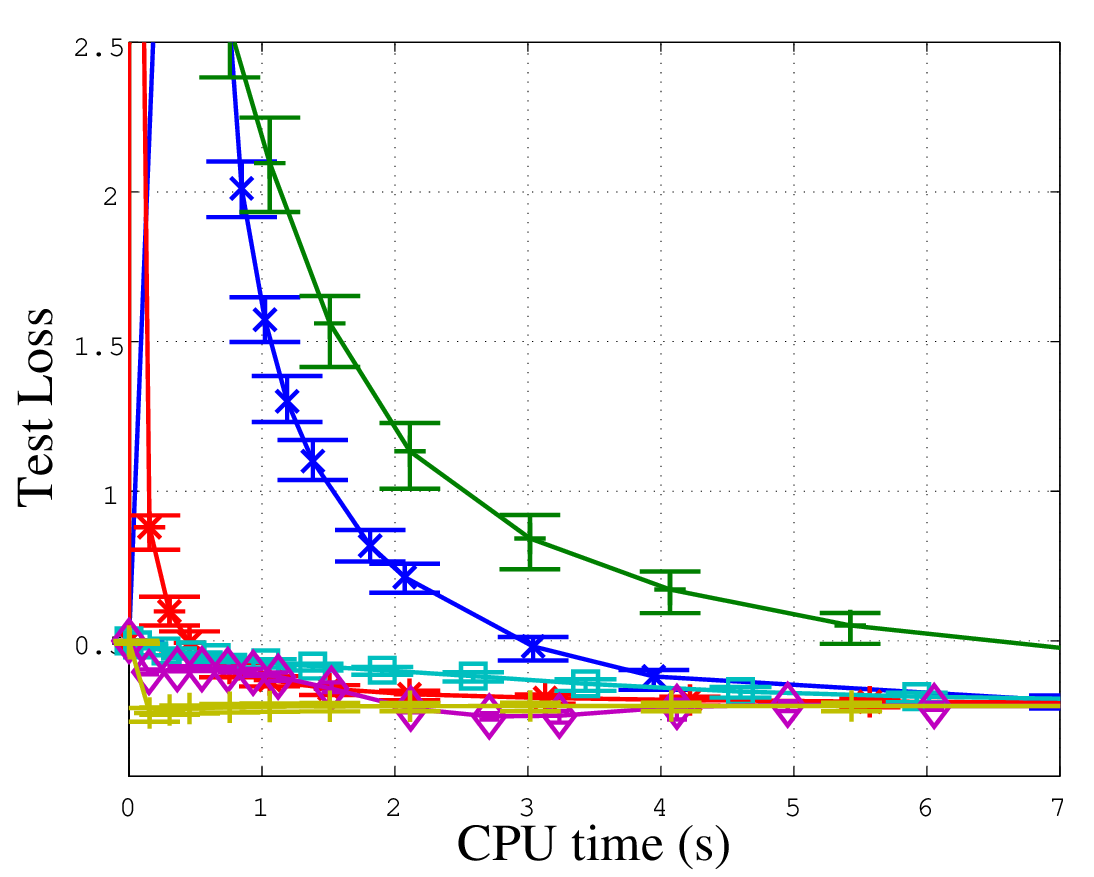} 
\includegraphics[width=\myfigwidth]{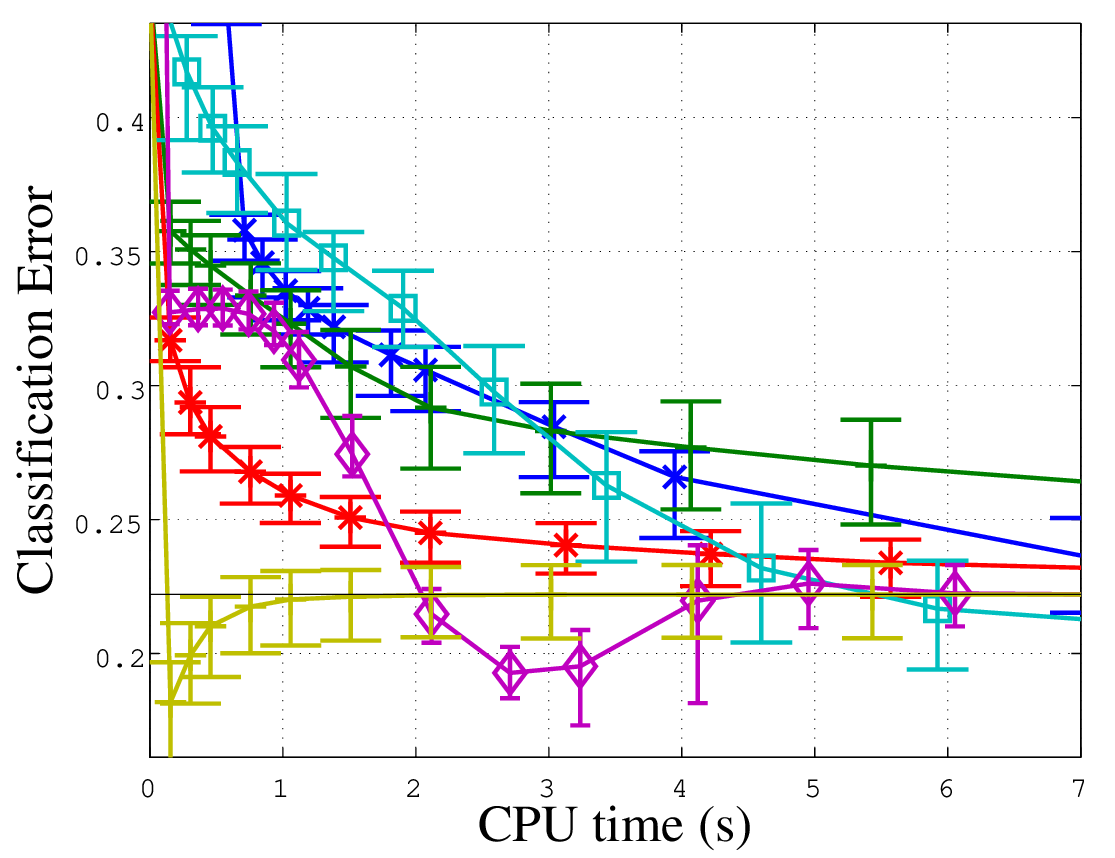}} \\
\subfigure[$n=5120$]{
\includegraphics[width=\myfigwidth]{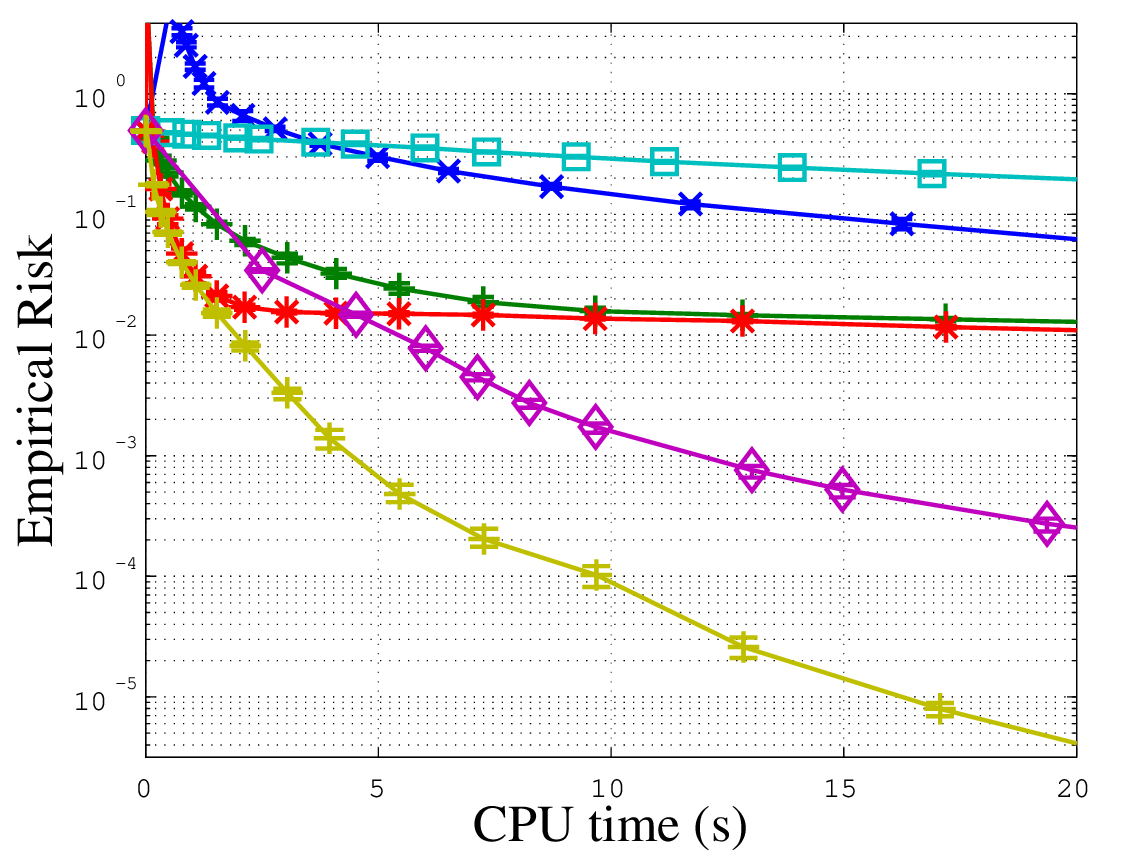} 
\includegraphics[width=\myfigwidth]{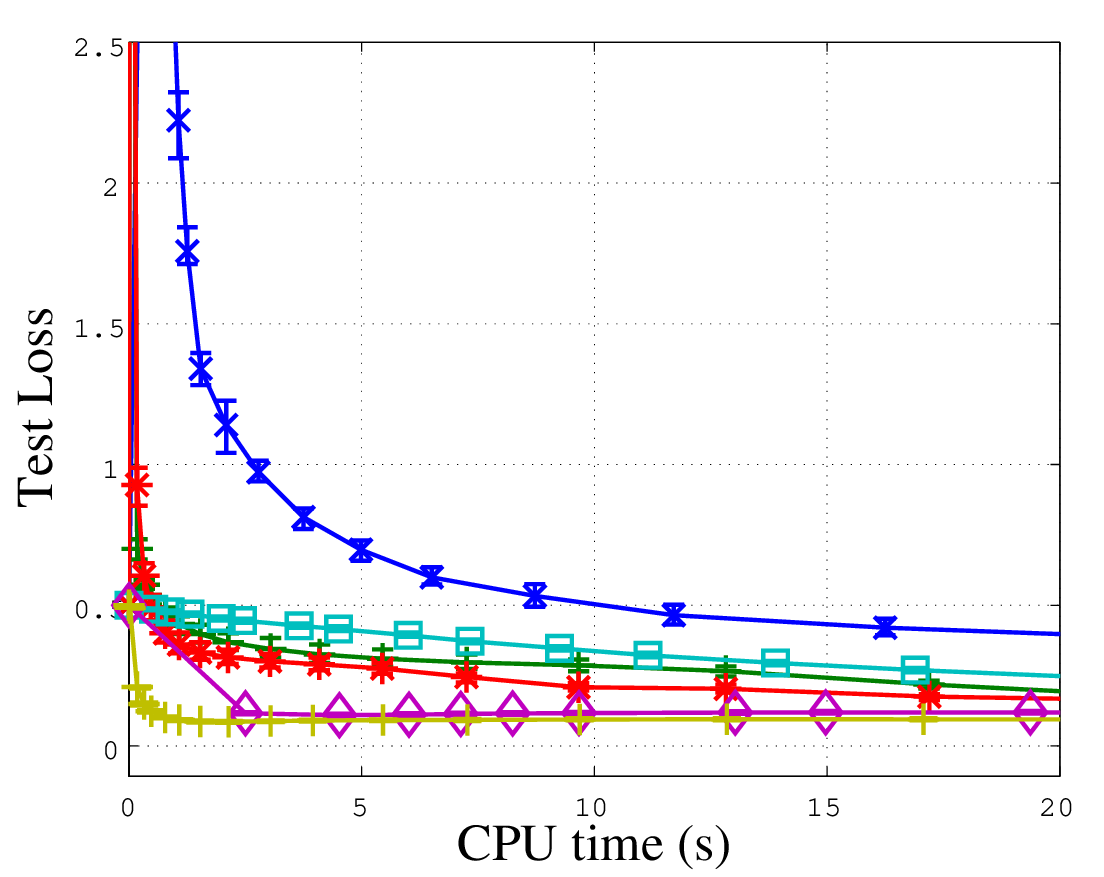} 
\includegraphics[width=\myfigwidth]{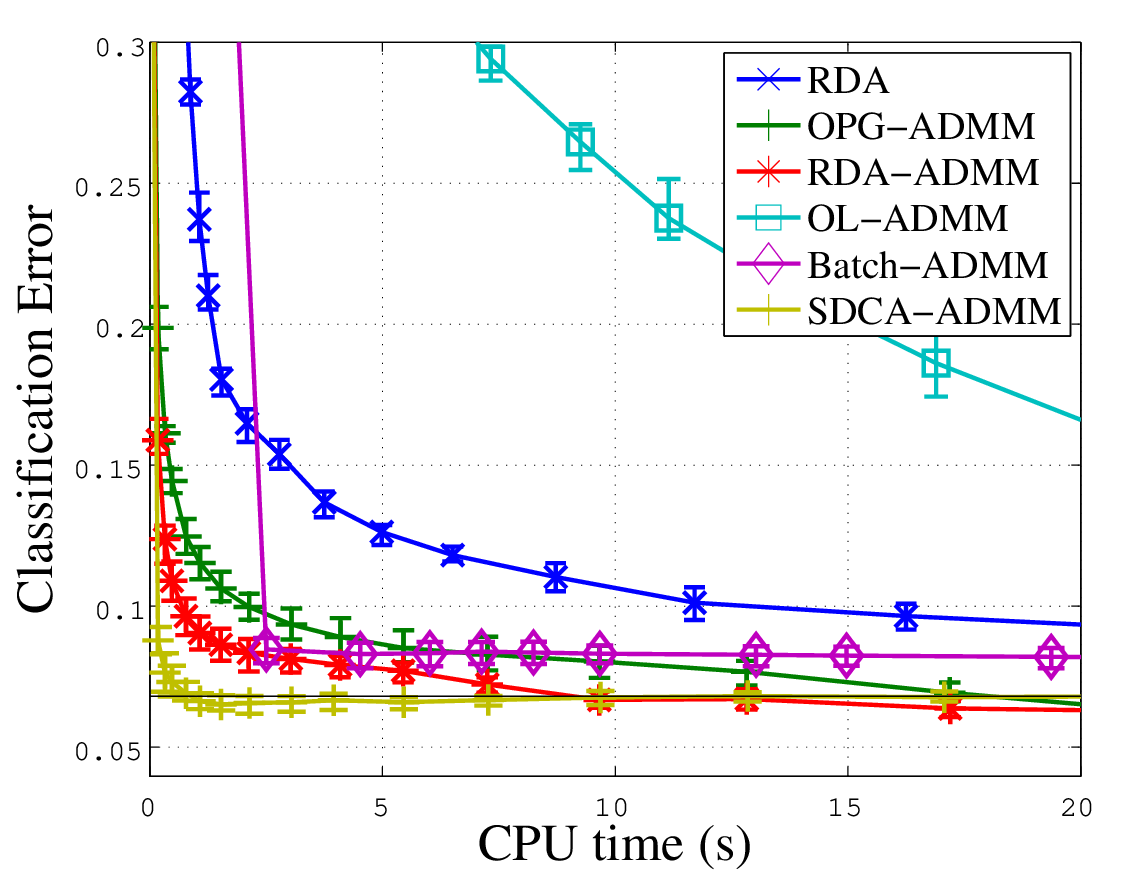}
}
\caption{Artificial data: Excess empirical risk, exected loss on the test data and test classification error averaged over 10 independent iteration against CPU time in artificial data with (a) $n=512$
and (b) $n=5120$.
The error bar indicates the standard deviation.
}
\label{fig:ArtificialData}
\end{center}
\end{figure}

We observe that the excess empirical risk of our method, SDCA-ADMM, actually converges linearly while other stochastic methods don't show linear convergence.
Although Batch-ADMM also shows linear convergence and its convergence speed is comparable to SDCA-ADMM for small sample situation ($n=512$), 
SDCA-ADMM is much faster than Batch-ADMM when the number of samples is large ($n=5120$).
As for the classification error, existing stochastic methods also show nice performances despite the poor convergence of the empirical risk.
On the other hand, SDCA-ADMM rapidly converges to a stable state and shows comparable or better classification accuracy than existing methods.

\subsection{Real Data}
Here we execute numerical experiments on real data sets;
`20 Newsgroups'\footnote{Available at http://www.cs.nyu.edu/\~{}roweis/data.html. 
We converted the four class classification task into binary classification by grouping category 1,2 and category 3,4 respectively.} and `a9a'\footnote{Available at `LIBSVM data sets' http://www.csie.ntu.edu.tw/\~{}cjlin/libsvmtools/datasets.}. 
`20 Newsgroups' contains 100 dimensional 12,995 training samples and 3,247 test samples.
`a9a' contains 123 dimensional  32,561 training samples and 16,281 test samples.
We constructed a similarity graph between features using graph Lasso and applied graph guided regularization as in \citet{ICML:Ouyang+etal:2013}.
That is, we applied graph Lasso to the training samples, and obtain a sparse inverse variance-covariance matrix $\hat{F}$.
Based on the similarity matrix $\hat{F}$, we connect all index pairs $(i,j)$ with $\hat{F}_{i,j} \neq 0$ on edges. 
We denote by $E$ the set of edges. 
Then we impose the following graph guided regularization:
\begin{align*}
&\textstyle \tilde{\psi}(w) = C_1 \sum_{i=1}^p |w_i| + C_2 \sum_{(i,j) \in E} |w_i - w_j| + 0.01\times (C_1\sum_{i=1}^p |w_i|^2 + C_2\sum_{(i,j) \in E} |w_i - w_j|^2).
\end{align*}
Now let $F$ be $|E| \times p$ matrix where $F_{e,i} = 1$ and $F_{e,j} = -1$, if $(i,j) = e \in E$, and $F_{e,i} = 0$ otherwise.
Then by letting $B^\top =[\eyemat{p};F]$ and $\psi(u) = C_1 \sum_{i=1}^{p} |u_i| + C_2 \sum_{i=p+1}^{|E|} |u_i|
+0.01(C_1 \sum_{i=1}^{p} |u_i|^2 + C_2 \sum_{i=p+1}^{|E|} |u_i|^2)$ for $u \in \Real^{p+|E|}$, 
we have $\tilde{\psi}(w) = \psi(B^\top w)$. Note that the proximal operation with respect to $\psi$ is just the soft-thresholding operation.
In our experiments, we employed $C_2 = C_1 |E|/p$ and $C_1 = 0.01/\sqrt{n}$.


\begin{figure}[h]
\begin{center}
\subfigure[20news]{
\includegraphics[width=\myfigwidth]{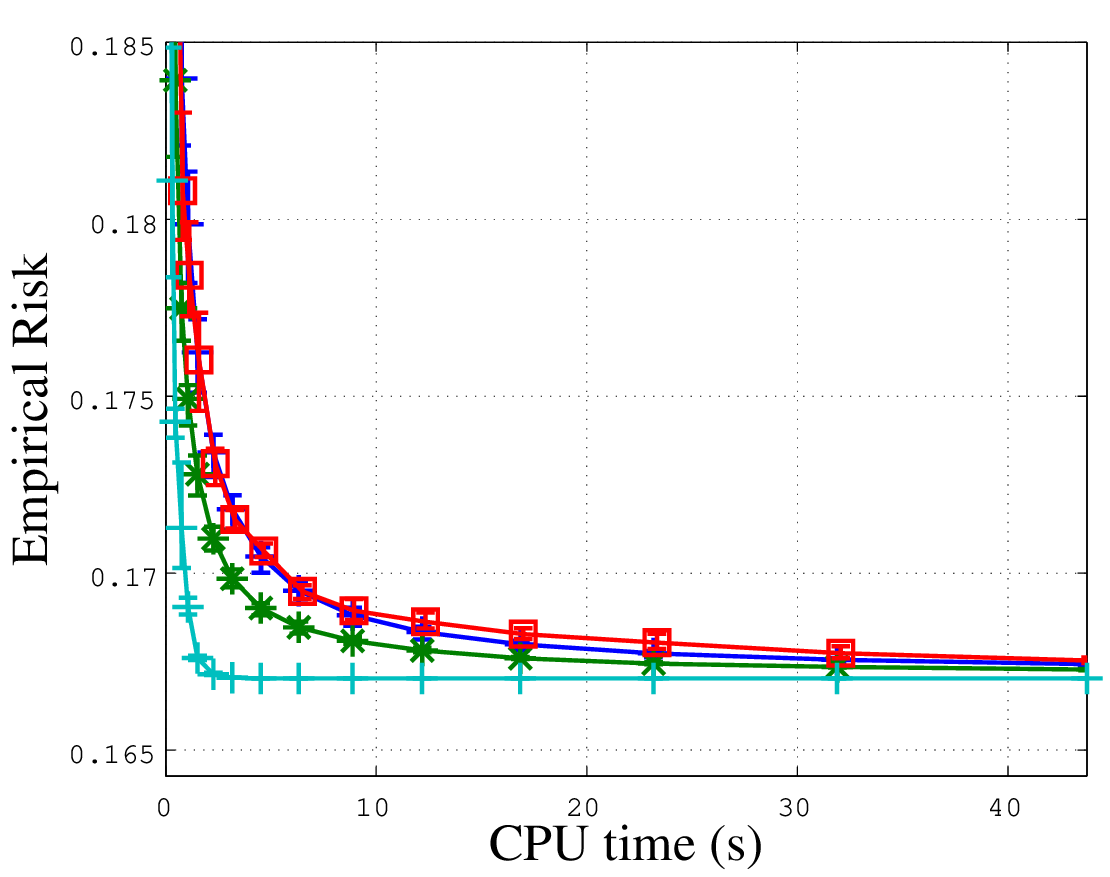} 
\includegraphics[width=\myfigwidth]{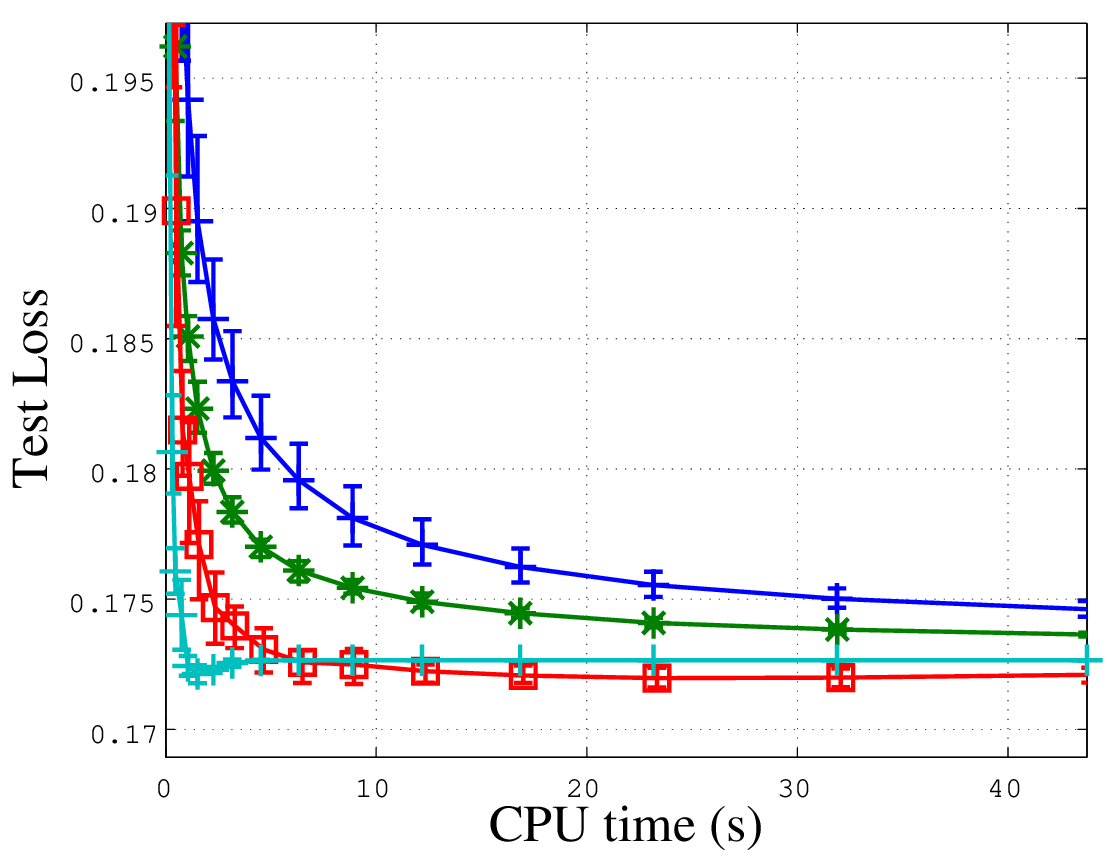} 
\includegraphics[width=\myfigwidth]{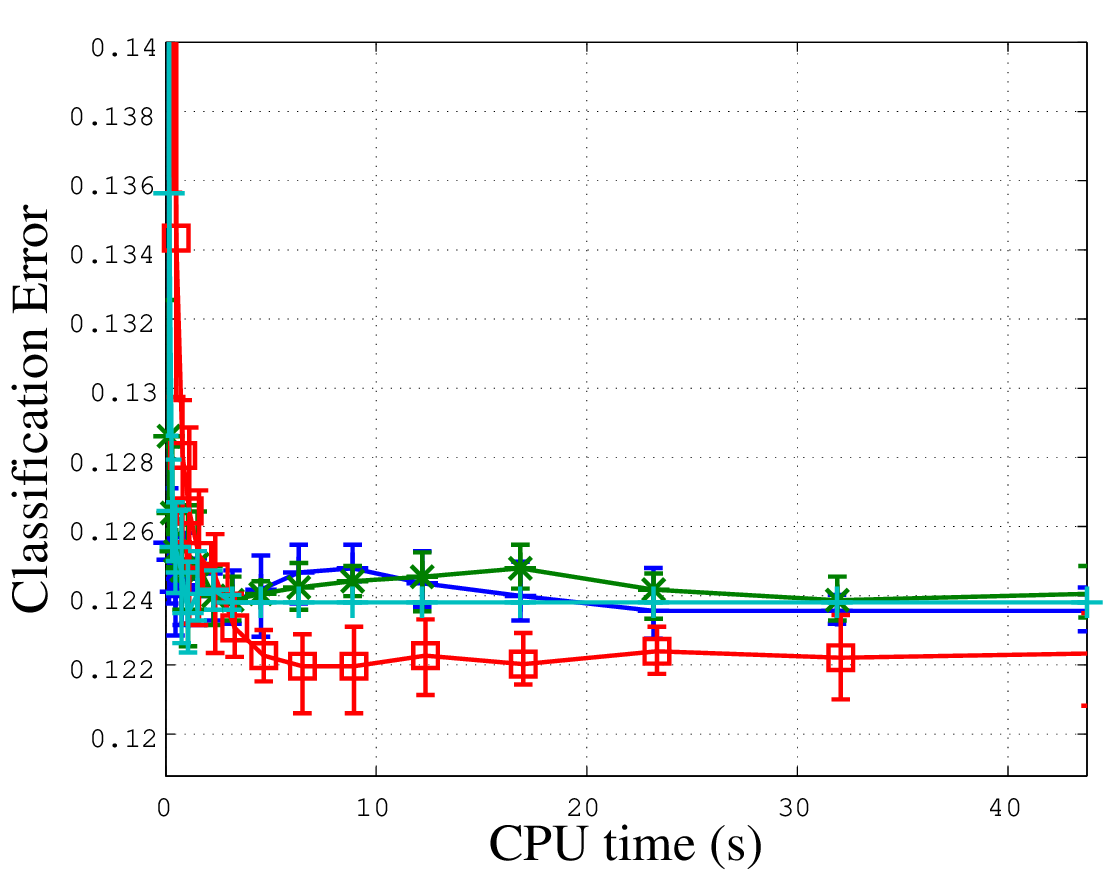} 
}
\\
\subfigure[a9a]{
\includegraphics[width=\myfigwidth]{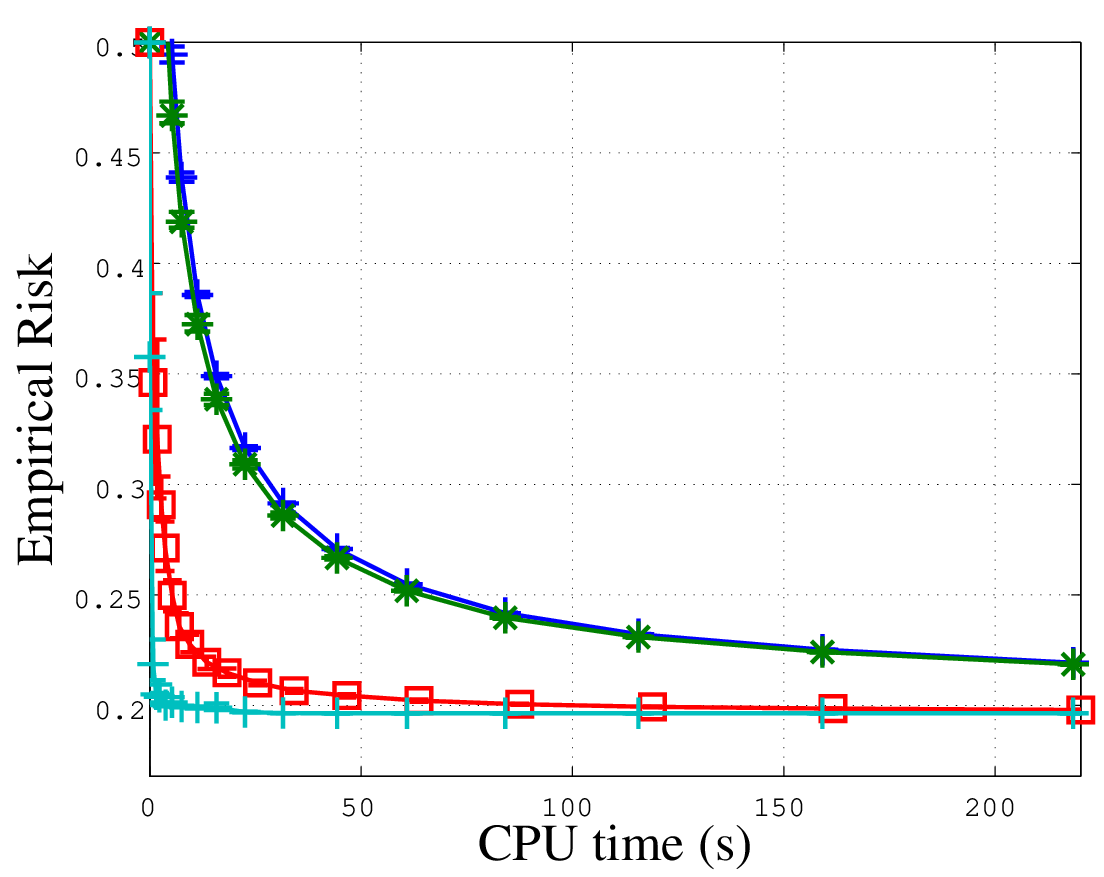} 
\includegraphics[width=\myfigwidth]{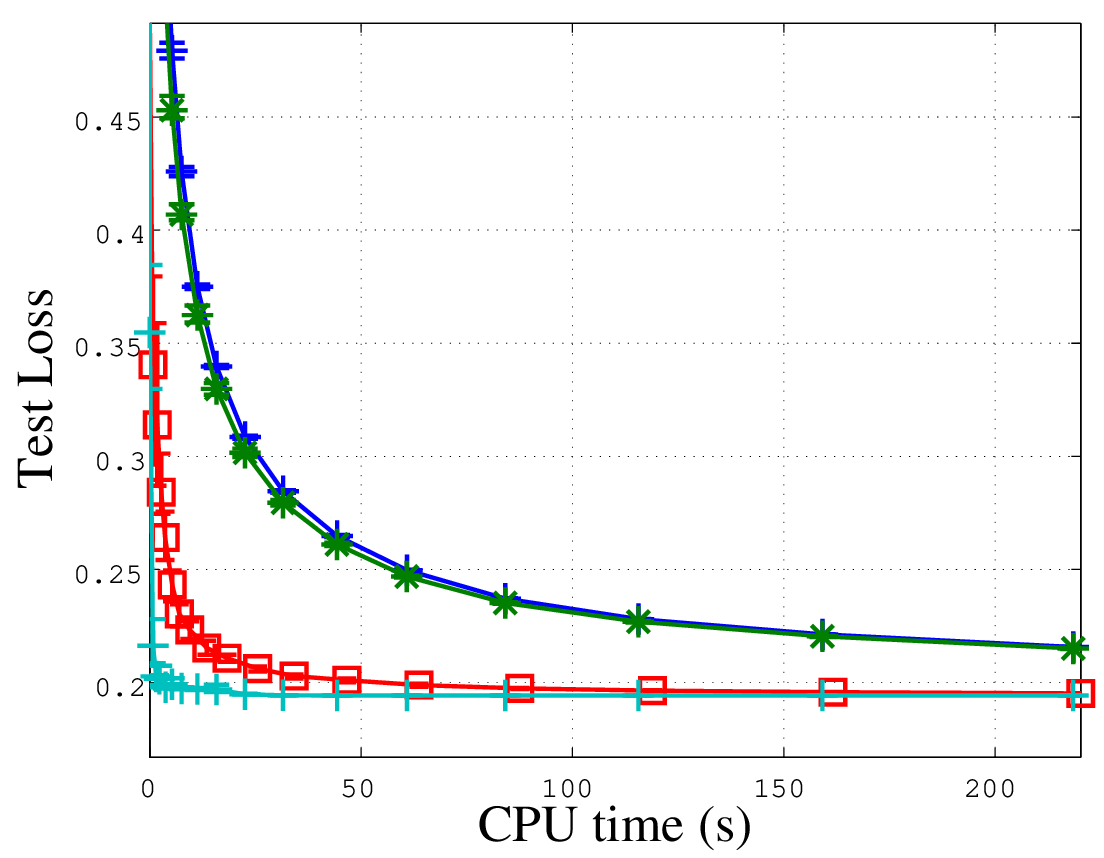} 
\includegraphics[width=\myfigwidth]{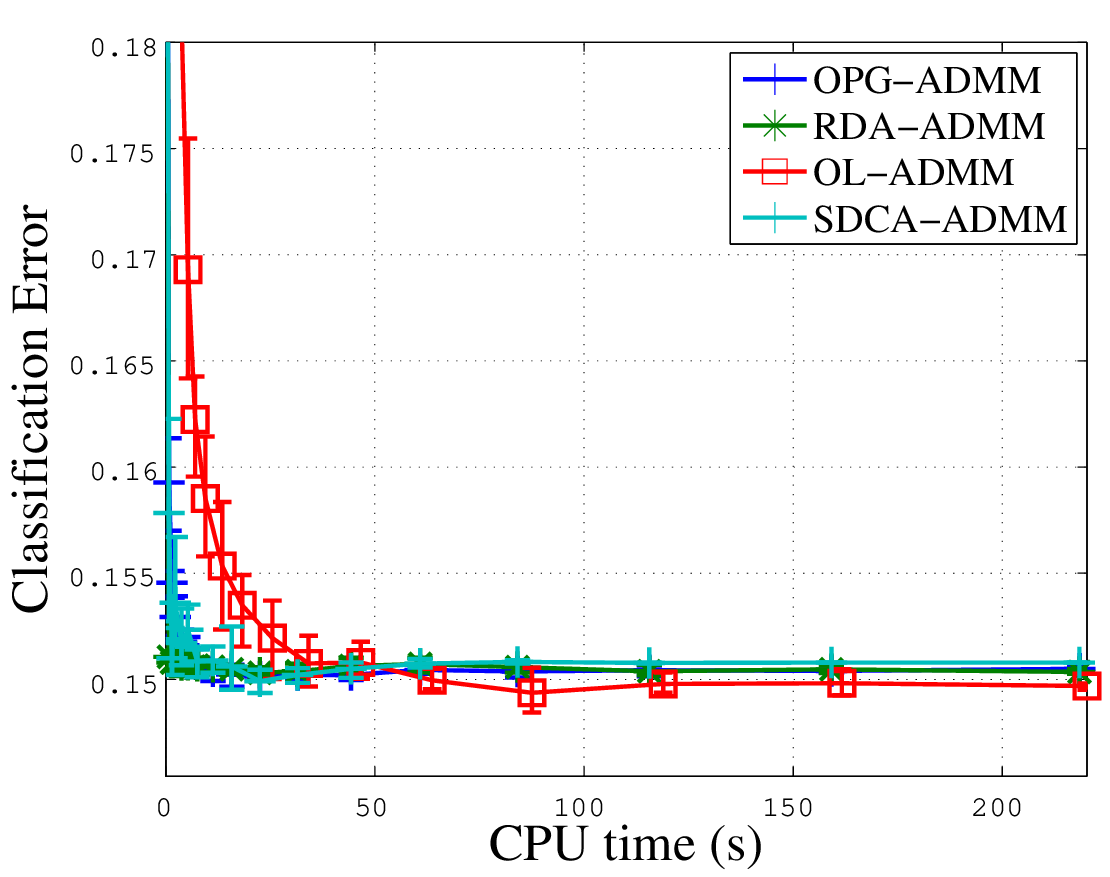} 
}
\caption{Real data: Empirical risk, average loss on the test data and test classification error averaged over 5 independent iteration against CPU time in real data
((a) 20news, (b) a9a).
The error bar indicates the standard deviation.}
\label{fig:RealData}
\end{center}
\end{figure}

We computed the empirical risk on the training data, the averaged loss on the test data, and the test classification error (Figure \ref{fig:RealData}).
We observe that the empirical risk on the training data of SDCA-ADMM converges much faster than other methods.
Although other methods also performs well on the test loss and the classification error,
SDCA-ADMM still converges faster than existing methods with respect to the two quantities measured on the test data.

\section{Conclusion}
We proposed a new stochastic dual coordinate ascent technique with alternating direction multiplier method.
The proposed method can be applied to wide range of regularization functions.
Moreover, we proposed a mini-batch extension of our method.
It is shown that, under some strong convexity conditions, our method converges exponentially.
According to our analysis, the mini-batch method improves the convergence rate if the input features don't have strong correlation between each other.
The numerical experiments showed that our method actually converges exponentially,
and the convergence is fast in terms of both empirical and expected risk.

Future work includes that the determination of $\eta_{Z,I}$.
In Theorem \ref{th:LinearConvergence}, 
the exponential convergence is guaranteed if $\eta_{Z,I} >= (1+2\gamma n(1-1/K))\sigma_{\max}(Z_I^\top Z_I)$.
However, in our preliminary numerical experiments, an aggressive method like the one suggested in \citet{ICML:Takac+etal:2013} performed effectively in some data sets.
Developing more sophisticated determination of $\eta_{Z,I}$ (and $G$) would be a potentially promising future work.

{ 
\bibliographystyle{abbrvnat}
\bibliography{main} 

\begin{thebibliography}{24}
\providecommand{\natexlab}[1]{#1}
\providecommand{\url}[1]{\texttt{#1}}
\expandafter\ifx\csname urlstyle\endcsname\relax
  \providecommand{\doi}[1]{doi: #1}\else
  \providecommand{\doi}{doi: \begingroup \urlstyle{rm}\Url}\fi

\bibitem[Boyd et~al.(2010)Boyd, Parikh, Chu, Peleato, and
  Eckstein]{FTML:Boyd+etal:2010}
S.~Boyd, N.~Parikh, E.~Chu, B.~Peleato, and J.~Eckstein.
\newblock Distributed optimization and statistical learning via the alternating
  direction method of multipliers.
\newblock \emph{Foundations and Trends in Machine Learning}, 3:\penalty0
  1--122, 2010.

\bibitem[Deng and Yin(2012)]{TechRepo:Deng+Yin:2012}
W.~Deng and W.~Yin.
\newblock On the global and linear convergence of the generalized alternating
  direction method of multipliers.
\newblock Technical report, Rice University CAAM TR12-14, 2012.

\bibitem[Duchi and Singer(2009)]{JMLR:Duchi+Singer:2009}
J.~Duchi and Y.~Singer.
\newblock Efficient online and batch learning using forward backward splitting.
\newblock \emph{Journal of Machine Learning Research}, 10:\penalty0 2873--2908,
  2009.

\bibitem[Gabay and Mercier(1976)]{CMA:Gabay+Mercier:1976}
D.~Gabay and B.~Mercier.
\newblock A dual algorithm for the solution of nonlinear variational problems
  via finite-element approximations.
\newblock \emph{Computers \& Mathematics with Applications}, 2:\penalty0
  17--40, 1976.

\bibitem[Hestenes(1969)]{JOTA:Hestenes:1969}
M.~Hestenes.
\newblock Multiplier and gradient methods.
\newblock \emph{Journal of Optimization Theory \& Applications}, 4:\penalty0
  303--320, 1969.

\bibitem[Jacob et~al.(2009)Jacob, Obozinski, and Vert]{ICML:Jacob+etal:2009}
L.~Jacob, G.~Obozinski, and J.-P. Vert.
\newblock Group lasso with overlap and graph lasso.
\newblock In \emph{Proceedings of the 26th International Conference on Machine
  Learning}, 2009.

\bibitem[{Le Roux} et~al.(2013){Le Roux}, Schmidt, and
  Bach]{NIPS:LeRoux+Schmidt+Bach:2012}
N.~{Le Roux}, M.~Schmidt, and F.~Bach.
\newblock A stochastic gradient method with an exponential convergence rate for
  strongly-convex optimization with finite training sets.
\newblock In \emph{Advances in Neural Information Processing Systems 25}, 2013.

\bibitem[Nemirovskii and Yudin(1983)]{Book:Nemirovskii+Yudin:1983}
A.~Nemirovskii and D.~Yudin.
\newblock \emph{Problem complexity and method efficiency in optimization}.
\newblock John Wiley, New York, 1983.

\bibitem[Ouyang et~al.(2013)Ouyang, He, Tran, and Gray]{ICML:Ouyang+etal:2013}
H.~Ouyang, N.~He, L.~Q. Tran, and A.~Gray.
\newblock Stochastic alternating direction method of multipliers.
\newblock In \emph{Proceedings of the 30th International Conference on Machine
  Learning}, 2013.

\bibitem[Powell(1969)]{Opt:Powell:1969}
M.~Powell.
\newblock A method for nonlinear constraints in minimization problems.
\newblock In R.~Fletcher, editor, \emph{Optimization}, pages 283--298. Academic
  Press, London, New York, 1969.

\bibitem[Qin and Goldfarb(2012)]{JMLR:Qin+Goldfarb:2012}
Z.~Qin and D.~Goldfarb.
\newblock Structured sparsity via alternating direction methods.
\newblock \emph{Journal of Machine Learning Research}, 13:\penalty0 1435--1468,
  2012.

\bibitem[Rakotomamonjy(2013)]{NeuroC:Rakotomamonjy:2013}
A.~Rakotomamonjy.
\newblock Applying alternating direction method of multipliers for constrained
  dictionary learning.
\newblock \emph{Neurocomputing}, 106:\penalty0 126--136, 2013.

\bibitem[Rockafellar(1970)]{Book:Rockafellar:ConvexAnalysis}
R.~T. Rockafellar.
\newblock \emph{Convex Analysis}.
\newblock Princeton University Press, Princeton, 1970.

\bibitem[Rockafellar(1976)]{Roc76}
R.~T. Rockafellar.
\newblock Augmented {L}agrangians and applications of the proximal point
  algorithm in convex programming.
\newblock \emph{Mathematics of Operations Research}, 1:\penalty0 97--116, 1976.

\bibitem[Shalev-Shwartz and
  Zhang(2013{\natexlab{a}})]{JMLR:SDCA:Shai+Tong:2013}
S.~Shalev-Shwartz and T.~Zhang.
\newblock Stochastic dual coordinate ascent methods for regularized loss
  minimization.
\newblock \emph{Journal of Machine Learning Research}, 14:\penalty0 567--599,
  2013{\natexlab{a}}.

\bibitem[Shalev-Shwartz and Zhang(2013{\natexlab{b}})]{NIPS:Shwartz+Zhang:2013}
S.~Shalev-Shwartz and T.~Zhang.
\newblock Accelerated mini-batch stochastic dual coordinate ascent.
\newblock In \emph{Advances in Neural Information Processing Systems 26},
  2013{\natexlab{b}}.

\bibitem[Shalev-Shwartz and Zhang(2013{\natexlab{c}})]{arXiv:Shai+Tong:2012}
S.~Shalev-Shwartz and T.~Zhang.
\newblock Proximal stochastic dual coordinate ascent.
\newblock Technical report, 2013{\natexlab{c}}.
\newblock arXiv:1211.2717.

\bibitem[Signoretto et~al.(2010)Signoretto, Lathauwer, and
  Suykens]{TechRepo:Signoretto+etal:2010}
M.~Signoretto, L.~D. Lathauwer, and J.~Suykens.
\newblock Nuclear norms for tensors and their use for convex multilinear
  estimation.
\newblock Technical Report 10-186, ESAT-SISTA, K.U.Leuven, 2010.

\bibitem[Suzuki(2013)]{ICML:Suzuki:2013}
T.~Suzuki.
\newblock Dual averaging and proximal gradient descent for online alternating
  direction multiplier method.
\newblock In \emph{Proceedings of the 30th International Conference on Machine
  Learning}, volume~28, pages 392--400. JMLR Workshop and Conference
  Proceedings, 2013.

\bibitem[Tak{\'a}{\v c} et~al.(2013)Tak{\'a}{\v c}, Bijral, Richt{\'a}rik, and
  Srebro]{ICML:Takac+etal:2013}
M.~Tak{\'a}{\v c}, A.~Bijral, P.~Richt{\'a}rik, and N.~Srebro.
\newblock Mini-batch primal and dual methods for {SVM}s.
\newblock In \emph{the 30th International Conference on Machine Learning},
  2013.

\bibitem[Tomioka et~al.(2011)Tomioka, Suzuki, Hayashi, and
  Kashima]{NIPS:Tomioka+etal:2011}
R.~Tomioka, T.~Suzuki, K.~Hayashi, and H.~Kashima.
\newblock Statistical performance of convex tensor decomposition.
\newblock In \emph{Advances in Neural Information Processing Systems 25}, 2011.

\bibitem[Wang and Banerjee(2012)]{ICML:Wang+Banerjee:2012}
H.~Wang and A.~Banerjee.
\newblock Online alternating direction method.
\newblock In \emph{Proceedings of the 29th International Conference on Machine
  Learning}, 2012.

\bibitem[Xiao(2009)]{NIPS:Xiao:2009}
L.~Xiao.
\newblock Dual averaging methods for regularized stochastic learning and online
  optimization.
\newblock In \emph{Advances in Neural Information Processing Systems 23}, 2009.

\bibitem[Yuan et~al.(2011)Yuan, Liu, and Ye]{NIPS:Yuan+etal:FOGLASSO:2011}
L.~Yuan, J.~Liu, and J.~Ye.
\newblock Efficient methods for overlapping group lasso.
\newblock In \emph{Advances in Neural Information Processing Systems 24}, 2011.

\end{thebibliography}
}

\appendix 

\section{Appendix: Proof of Theorem \ref{th:LinearConvergence}}

Here, we give the proof of Theorem \ref{th:LinearConvergence}.
For notational simplicity, we rewrite the dual problem as follows:
\begin{subequations}
\label{eq:GeneralOptim}
\begin{align}
\min_{x \in \calX,y \in \calY} &~~\sum_{i=1}^n g_i(x_i) + \phi(y), \\
\mathrm{s.t.} &~~Zx+By = 0,
\end{align}
\end{subequations}
where $Z\in \Real^{p \times n}$, $B \in \Real^{p \times d}$.
This is equivalent to the dual optimization problem in the main text
when
$g_i = f_i^*$ and $\phi = n \psi^*(\cdot/n)$ (or equivalently $\phi^* = n\psi$). 
We write $g(x) = \sum_{i=1}^n g_i(x_i)$.

Then we consider the following update rule:
\begin{align*}
\yt{t} & \leftarrow \mathop{\arg \min}_{y}  \phi(y) -\langle \lambdat{t-1}, Z\xt{t-1} + By \rangle + \frac{\rho}{2}\|Z\xt{t-1} + By\|^2
+\frac{1}{2}\|y - \yt{t-1}\|_{Q} \\
\xt{t}_i & \leftarrow \mathop{\arg \min}_{x_I} \sum_{i\in I}g_i(x_i) - \langle \lambdat{t-1}, Z_Ix_I + B\yt{t}\rangle + \frac{\rho}{2}\|Z_Ix_I + Z_{\backslash I} \xt{t-1}_{\backslash I} + B\yt{t}\|^2
+\frac{1}{2}\|x_I - \xt{t-1}_I\|_{G_{ii}} \\
\lambdat{t} &= \lambdat{t-1} - \gamma\rho\{n(Z\xt{t} + B\yt{t})- (n-n/K)(Z\xt{t-1} + B\yt{t-1})\}.
\end{align*}

Assumption \ref{ass:LocalStrongConvexity} can be interpreted as follows.
There is an optimal solution $(\xstar,\ystar)$ and corresponding Lagrange multiplier $\lambdastar$ such that 
$$
\nabla g(\xstar) = Z^\top \lambdastar,~~\nabla \phi(\ystar) \ni B^\top \lambdastar.
$$
Moreover, we suppose that each (dual) loss function $g_i$ is $v$-strongly convex and $\phi$ is $h$-smooth:
\begin{align*}
&g_i(x_i) - g_i(x_i^*) \geq \langle \nabla g_i(x_i^*), x_i - x_i^* \rangle + \frac{v\|x_i-x_i^*\|^2}{2}. 
\end{align*}
We also assume that there exit $h$ and $v_\phi$ such that, for all $y,u$ and all $\ystar \in \calY^*$, there exits $\widehat{y}^* \in \calY^*$ which depends on $y$ and we have 
\begin{align*}
&\phi(y) - \phi(\ystar) \geq \langle B^\top \lambdastar, y - \ystar \rangle + \frac{v_{\phi}'}{2}\|P_{\Ker(B)}(y - \ystar)\|^2, \\
&\phi^*(u) - \phi^*(B^\top \lambdastar) \geq \langle \ystar,  u - B^\top \lambdastar \rangle + \frac{h'}{2}\|u - B^\top \lambdastar \|^2.
\end{align*}
Note that the primal and dual are flipped compared with the main text.
Once can check that there is a correspondence between $v_\psi,h$ in the main text and $v_{\phi}'$ and $h'$ 
such that $v_{\phi}' = \frac{v_{\psi}}{n}$ and $h' = n h$.

Define 
$$
F(x,y) := \sum_{i=1}^n g_i(x_i) + \phi(y) - \langle \lambdastar, Zx + By \rangle~~~(=n F_D(x,y)).
$$
By the definition of $\lambdastar$, one can easily check that 
$$
F(x,y) - F(\xstar,\ystar) \geq \frac{nv}{2} \|x - \xstar\|^2 \geq 0.
$$
We define 
\begin{align*}
&R'(x,y,w) \\
= & F(x,y) - F(\xstar,\ystar)
+ \frac{2}{\rho}\|\lambdat{t} - \lambdastar\|^2 
+\frac{\rho(1-\gamma)}{2}\|Zx+By\|^2
+ \frac{1}{2}\|x - \xstar\|^2_{vI_p + H} + \frac{1}{2n}\|y-\calY^*\|^2_Q.
\end{align*}
Here again we have that $R' = nR_D$.
Let $\nhat = n/K$, the expected cardinality of $|I|$, and let $\DiagI(S)$ be a block diagonal matrix whose $I_k \times I_k~(k=1,\dots,K)$ diagonal elements are non-zero and given by 
$(\Diag(S))_{I_k,I_k} = S_{I_k,I_k}$ ($k=1,\dots,K$).

\begin{Theorem}
\label{th:SuppleConvergenceDual}
Suppose that $\gamma = \frac{1}{4n}$, $\DiagI(G) \succ 2\gamma \rho(n-1)\DiagI(Z^\top Z)$ and $B^\top$ is injective.
Then, under the assumptions, the objective function converges R-linearly:
\begin{align*}
&R'(\xt{t},\yt{t},\wt{t}) \leq  \left(1-\frac{\mu}{K}\right)^T R(\xt{0},\yt{0},\wt{0}), \\
&\EE[F(\xt{t},\yt{t}) - F(x^*,y^*)] \leq \left(1-\frac{\mu}{K}\right)^T R(\xt{0},\yt{0},\wt{0}),
\end{align*}
where
\begin{align*}
\mu := \min &\Bigg\{ \frac{1}{2}\left(\frac{v}{v+\sigma_{\max}(H)} \right),
\frac{h'\rho\sigma_{\min}(B B^\top)}{2\max\{1,4h'\rho,4h'\sigma_{\max(Q)}\}}, \frac{n v_\phi'}{4\sigma_{\max(Q)}}, 
\frac{n v\sigma_{\min}(BB^\top)}{\sigma_{\max(Q)} (\rho \sigma_{\max}(Z^\top Z) + 4v)}   \Bigg\},
\end{align*}
In particular, we have that
$$
\EE[ \|\lambdat{t} - \lambdastar\|^2 ] \leq \frac{\rho}{2} \left(1-\frac{\mu}{K}\right)^T R(\xt{0},\yt{0},\wt{0}).
$$

\end{Theorem}
Theorem 1 in the main text can be obtained using the relation $v_{\phi}' = \frac{v_{\phi}}{n}$, $h' = n h$, $F = n F_D $ and $R' = nR_D$.
The convergence of the primal objective is obtained by using the following fact:
Since $g$ is strongly convex, we have that
\begin{align*}
& g(x) - g(x^*) \geq \langle \nabla g(x^*), x - x^* \rangle + \frac{v\|x-x^*\|^2}{2}~~(\forall x) \\
\Rightarrow
& g^*(u)  \leq g^*(u^*) + \langle \nabla g^*(u^*), u - u^* \rangle + \frac{\|u-u^*\|^2}{2v}~~(\forall v),
\end{align*}
where $u^* \in \nabla g(x^*)$.
Using this, we have that, 
\begin{align*}
\frac{1}{n} \sum_{i=1}^n f_i(z_i^\top \wt{t}) - \frac{1}{n} \sum_{i=1}^n f_i(z_i^\top w^*)
& \leq \left \langle Z \nabla \left(\frac{1}{n} \sum_{i=1}^n f_i\right)(u)\big|_{u=Z^\top w^*}, \wt{t} - w^*  \right \rangle
+ \frac{\|Z^\top (\wt{t} - w^*)\|^2}{2nv} \\
& = \left\langle - y^*/n, B^\top (\wt{t} - w^*)  \right \rangle
+ \frac{\|Z^\top (\wt{t} - w^*)\|^2}{2nv},
\end{align*}
where we used the relation $Z x^* + B y^* = 0$. 
Moreover, using the relation $\psi(B^\top w) \leq \psi(B^\top w^*) + \langle y^*/n, B^\top(w- w^*)\rangle + l_1\|w- w^*\| + l_2\|w- w^*\|^2$ 
and the Jensen's inequality $\EE[\|\wt{T}- w^*\|]^2 \leq \EE[\|\wt{T}- w^*\|^2]$,
we obtain the assertion.


\begin{proof}[Proof of Theorem \ref{th:SuppleConvergenceDual}] 
~\\
\noindent {\it Step 1 {\rm (}Deriving a basic inequality{\rm)}}:
\begin{align}
 &g(\xt{t}) - g(\xt{t-1}) + \phi(\yt{t}) - \phi(\yt{t-1})  \notag\\ 
=&\sum_{i\in I}g_i(\xt{t}_i) - \sum_{i\in I} g_i(\xt{t-1}_i) + \phi(\yt{t}) - \phi(\yt{t-1})  \notag\\
=&\sum_{i\in I}g_i(\xt{t}_i) - \langle \lambdat{t-1}, Z\xt{t} + B\yt{t} \rangle + \frac{\rho}{2}\|Z\xt{t}+B\yt{t}\|^2 + \frac{1}{2} \|\xt{t}_I - \xt{t-1}_I\|_{G_{I,I}}^2  \notag\\
&+ \langle \lambdat{t-1}, Z\xt{t} + B\yt{t} \rangle - \frac{\rho}{2}\|Z\xt{t}+B\yt{t}\|^2 - \frac{1}{2} \|\xt{t}_I - \xt{t-1}_I\|_{G_{I,I}}^2  \notag \\
&-\sum_{i\in I}g_i(\xt{t-1}_i) + \phi(\yt{t}) - \phi(\yt{t-1}). \label{eq:CAFirstBound_pre} 
\end{align}
Here we define that $\tilde{Z}_I = [Z_{\backslash I} Z_I]$ and $\tilde{x} := \begin{bmatrix}\xt{t-1}_{\backslash I} \\ x_I \end{bmatrix}$ for a given $x_I$, and 
$$
\tilde{g}_I(x_I) := 
\sum_{i\in I }g_i(x_i) - \left \langle \lambdat{t-1}, \tilde{Z}_I \tilde{x} + B\yt{t} \right \rangle 
+ \frac{\rho}{2}\|\tilde{Z}_I \tilde{x} +B\yt{t}\|^2 + \frac{1}{2} \|x_I - \xt{t-1}_I\|_{G_{I,I}}^2.
$$
Then by the update rule of $\xt{t}$, we have that 
$$
\tilde{g}_I(\xt{t}_I) \leq \tilde{g}_I(\xstar_I) - \frac{v}{2}\|\xt{t}_I - \xstar_I\|^2 - \frac{\rho}{2}\|Z_I(\xt{t}_I - \xstar_I)\|^2 - \frac{1}{2}\|\xt{t}_I - \xstar_I \|_{G_{I,I}}, 
$$
which implies 
\begin{align*}
&\sum_{i \in I} g_i(\xt{t}_i) - \left \langle \lambdat{t-1}, Z \xt{t} + B\yt{t} \right \rangle 
+ \frac{\rho}{2}\|Z \xt + B\yt{t}\|^2 + \frac{1}{2} \|\xt{t}_I - \xt{t-1}_I\|_{G_{I,I}}^2 \\
\leq
&
\sum_{i \in I} g_i(\xstar_i) - \left \langle \lambdat{t-1}, Z_I \xstar_I + Z_{\backslash I} \xstar_{\backslash I} + B\yt{t} \right \rangle 
+ \frac{\rho}{2}\|Z_I \xstar_I + Z_{\backslash I} \xt{t-1}_{\backslash I} + B\yt{t}  \|^2 + \frac{1}{2} \|\xstar_I - \xt{t-1}_I\|_{G_{I,I}}^2 \\
& 
- \frac{v}{2}\|\xt{t}_I - \xstar_I\|^2 - \frac{\rho}{2}\|Z_I(\xt{t}_I - \xstar_I)\|^2 - \frac{1}{2}\|\xt{t}_I - \xstar_I \|_{G_{I,I}} \\
=
&
\sum_{i \in I} g_i(\xstar_i) - \left \langle \lambdat{t-1}, Z_I (\xstar_I - \xt{t}_I) \right \rangle
- \left \langle \lambdat{t-1}, Z \xt{t} + B\yt{t} \right \rangle  \\
&+ \frac{\rho}{2}\|Z_I \xstar_I + Z_{\backslash I} \xt{t-1}_{\backslash I} + B\yt{t}  \|^2
- \frac{\rho}{2}\|Z \xt{t} + B\yt{t}  \|^2
+ \frac{\rho}{2}\|Z \xt{t} + B\yt{t}  \|^2
 + \frac{1}{2} \|\xstar_I - \xt{t-1}_I\|_{G_{I,I}}^2 \\
&- \frac{v}{2}\|\xt{t}_I - \xstar_I\|^2 - \frac{\rho}{2}\|Z_I(\xt{t}_I - \xstar_I)\|^2 - \frac{1}{2}\|\xt{t}_I - \xstar_I \|_{G_{I,I}} \\
=
&
\sum_{i \in I} g_i(\xstar_i) - \left \langle \lambdat{t-1}, Z_I (\xstar_I - \xt{t}_I) \right \rangle \\
&- \frac{v}{2}\|\xt{t}_I - \xstar_I\|^2 - \frac{\rho}{2}\|Z_I(\xt{t}_I - \xstar_I)\|^2 - \frac{1}{2}\|\xt{t}_I - \xstar_I \|_{G_{I,I}} \\
&- \rho \langle  Z_{\backslash I} \xt{t}_{\backslash I} + B \yt{t}, Z_I (\xt{t}_I - \xstar_I) \rangle 
+ \frac{\rho}{2} \|Z_I \xstar_I \|^2 - \frac{\rho}{2} \|Z_I \xt{t}_I \|^2
 + \frac{1}{2} \|\xstar_I - \xt{t-1}_I\|_{G_{I,I}}^2 \\
&
- \left \langle \lambdat{t-1}, Z \xt{t} + B\yt{t} \right \rangle
+ \frac{\rho}{2}\|Z \xt{t-1} + B\yt{t}  \|^2.
\end{align*}
Using this, the RHS of \Eqref{eq:CAFirstBound_pre} can be further bounded by
\begin{align}
\text{(RHS)}\leq &\sum_{i\in I} g_i(\xstar_i) - \sum_{i\in I}g_i(\xt{t-1}_i)  - \langle \lambdat{t-1}, Z_I(\xstar_I - \xt{t}_I) \rangle  \notag \\
&- \frac{v}{2}\|\xt{t}_I - \xstar_I\|^2 - \frac{\rho}{2}\|Z_I(\xt{t}_I - \xstar_I)\|^2 - \frac{1}{2}\|\xt{t}_I - \xstar_I \|_{G_{I,I}} \notag \\
&- \rho \langle  Z_{\backslash I} \xt{t}_{\backslash I} + B \yt{t}, Z_I (\xt{t}_I - \xstar_I) \rangle 
+ \frac{\rho}{2} \|Z_I \xstar_I \|^2 - \frac{\rho}{2} \|Z_I \xt{t}_I \|^2  \notag\\
&+ \frac{1}{2} \|\xstar_I - \xt{t-1}_I\|_{G_{I,I}}^2 - \frac{1}{2} \|\xt{t}_I - \xt{t-1}_I\|_{G_{I,I}}^2  \notag \\
&+ \phi(\yt{t}) - \phi(\yt{t-1}). \label{eq:CAFirstBound} 
\end{align}
Here, we bound the term
$$
- \rho \langle  Z_{\backslash i} \xt{t}_{\backslash i} + B \yt{t}, Z_I (\xt{t}_I - \xstar_I) \rangle 
+ \frac{\rho}{2} \|Z_I \xstar_I \|^2 - \frac{\rho}{2} \|Z_I \xt{t}_I \|^2.
$$
By Lemma \ref{lemm:crosstermExpAxBy}, the expectation of this term is equivalent to 
\begin{align*}
&\EE\left[-\frac{\rho}{n} \langle Z\xt{t-1} + B\yt{t}, Z(n\xt{t}-(n-\nhat)\xt{t-1}-\nhat\xstar)\rangle\right] \\
&+\frac{\rho}{2K}\|\xt{t-1} - \xstar\|^2_{\DiagI(Z^\top Z)} - \frac{\rho}{2}\EE\left[\|\xt{t} - \xt{t-1}\|^2_{\DiagI(Z^\top Z)}\right].
\end{align*}

Note that, for any block diagonal matrix $S$ which satisfies $S_{I_k,I_{k'}} = (S_{i,j})_{(i,j)\in I_k \times I_{k'}} = O~(\forall k\neq k')$, we have that
\begin{align*}
\EE[\|\xt{t}_I - \xstar_I\|^2_{S_{I,I}}] 
& = \EE[\|\xt{t}_I - \xt{t-1}_I +  \xt{t-1}_I - \xstar_I\|^2_{S_{I,I}}] \\
& = \EE[\|\xt{t}_I - \xt{t-1}_I\|^2_{S_{I,I}}] + 
\EE[2 \langle \xt{t}_I - \xt{t-1}_I,  \xt{t-1}_I - \xstar_I\rangle_{S_{I,I}}] 
+ \EE[\|\xt{t-1}_I - \xstar_I\|^2_{S_{I,I}}] \\
& = \EE[\|\xt{t} - \xt{t-1}\|^2_S] + 
\EE[2 \langle \xt{t} - \xt{t-1},  \xt{t-1} - \xstar\rangle_S] 
+ \frac{1}{K} \|\xt{t-1} - \xstar\|^2_S \\
& = \EE[\|\xt{t} - \xstar\|^2_S] - \EE[\|\xt{t-1} - \xstar\|^2_S]
+ \frac{1}{K} \|\xt{t-1} - \xstar\|^2_S \\
&= \EE[\|\xt{t} - \xstar\|^2_S] - \left(1-\frac{1}{K}\right)\EE[\|\xt{t-1} - \xstar\|^2_S],
\end{align*}
where the expectation is taken with respect to the choice of $I \in \{I_1,\dots,I_K\}$.
Moreover, for a fixed vector $q$, we have that
\begin{align*}
&\EE[\langle q_I, \xt{t}_I - \xstar_I\rangle ] \\
= & \EE[\langle q_I, \xt{t}_I - \xt{t-1}_I  + \xt{t-1}_I  - \xstar_I\rangle ]  
=  \EE[\langle q, \xt{t} - \xt{t-1}\rangle]  + \EE[ \langle q_I, \xt{t-1}_I  - \xstar_I\rangle ]  \\
= & \EE[\langle q, \xt{t} - \xt{t-1}\rangle]  + \EE\left[ \sum_{k=1}^K 1[I=I_k] \langle q_{I_k}, \xt{t-1}_{I_k}  - \xstar_{I_k} \rangle \right] \\
= & \EE[\langle q, \xt{t} - \xt{t-1}\rangle]  + \frac{1}{K} \sum_{k=1}^K 1[I=I_k] \langle q_{I_k}, \xt{t-1}_{I_k}  - \xstar_{I_k}\rangle 
=  \EE[\langle q, \xt{t} - \xt{t-1}\rangle]  +\frac{1}{K} \langle q, \xt{t-1}  - \xstar \rangle \\ 
= & \EE\left[\left\langle q, \xt{t} - \left(1-\frac{1}{K}\right)\xt{t-1} - \frac{1}{K} \xstar \right\rangle\right]. 
\end{align*}
Then, by taking expectation with respect to $I$ and multiplying both sides of the above inequality by $n$, we have that
\begin{align}
&n \EE[g(\xt{t}) + \phi(\yt{t}) - g(\xt{t-1}) - \phi(\yt{t-1})] \notag \\
\leq &
g(\xstar) - g(\xt{t-1})  
+ \EE[\langle \lambdat{t-1},Z(n\xt{t} - (n-\nhat)\xt{t-1} - \nhat \xstar)\rangle] \notag\\
&- \EE\left[\frac{nv}{2}\|\xt{t} - \xstar\|^2 + \frac{n \rho}{2}\|\xt{t} - \xstar\|_{\DiagI(Z^\top Z)}^2 + \frac{n}{2}\|\xt{t} - \xstar \|_{\DiagI(G)}^2\right] \notag\\
&+ \frac{(n-\nhat)v}{2}\|\xt{t-1} - \xstar\|^2]  + \frac{(n-\nhat) \rho}{2}\|\xt{t-1} - \xstar\|_{\DiagI(Z^\top Z)}^2 + \frac{n-\nhat}{2}\|\xt{t-1} - \xstar \|_{\DiagI(G)}^2 \notag\\
&+\EE\left[-\rho \langle Z\xt{t-1} + B\yt{t}, Z(n\xt{t}-(n-\nhat)\xt{t-1}-\nhat\xstar)\rangle\right] \notag\\
&+\frac{\rho\nhat}{2}\|\xt{t-1} - \xstar\|^2_{\DiagI(Z^\top Z)} - \frac{n \rho}{2}\EE\left[\|\xt{t} - \xt{t-1}\|^2_{\DiagI(Z^\top Z)}\right] \notag\\
&+\frac{\nhat}{2}\|\xt{t-1} - \xstar\|^2_{\DiagI(G)} - \frac{n}{2}\EE[\|\xt{t}-\xt{t-1}\|^2_{\DiagI(G)}] \notag\\
&+ n \phi(\yt{t}) - n \phi(\yt{t-1}).
\label{eq:FDbound_beforephibound}
\end{align}
Here, note that the last two term $n \phi(\yt{t}) - n \phi(\yt{t-1})$ is bounded as 
\begin{align*}
&n \phi(\yt{t}) - n \phi(\yt{t-1}) \\
=
&\nhat (\phi(\yt{t}) - \phi(\yt{t-1})) + (n-\nhat)( \phi(\yt{t}) -  \phi(\yt{t-1})) \\
\leq
&\nhat(\phi(\ystar) - \phi(\yt{t-1})) + \left\langle \nabla \phi(\yt{t}), (n-\nhat)(\yt{t} - \yt{t-1}) + \nhat(\yt{t}- \ystar) \right\rangle \\
& -\frac{\nhat h'}{2}\|B^\top \lambdastar - \nabla \phi(\yt{t}) \|^2.
\end{align*}
for arbitrary $\ystar \in \calY^*$ where we used Lemma \ref{lemm:psidifbound} in the last line.
Define 
$$
\lambdatilde{t} := \lambdat{t-1} - \rho(Z\xt{t-1}+B\yt{t}).
$$
Note that $B^\top \lambdatilde{t} - Q(\yt{t} - \yt{t-1}) \in \nabla \phi(\yt{t})$.

Next, adding $\EE[ n \langle \lambdastar, Z(\xt{t-1} - \xt{t}) + B(\yt{t-1} - \yt{t})\rangle]$ to the both sides of \Eqref{eq:FDbound_beforephibound}, we have that
\begin{align}
&n \EE[F(\xt{t},\yt{t}) - F(\xt{t-1},\yt{t-1})] \notag\\
\leq
&
\nhat(F(\xstar,\ystar) - F(\xt{t-1},\yt{t-1})) \notag \\
&+ \EE[\langle \lambdat{t-1} - \lambdastar,Z(n\xt{t} - (n-\nhat)\xt{t-1} - \nhat \xstar)\rangle] \notag\\
&+ \EE[\langle \lambdatilde{t} - \lambdastar,B(n\yt{t} - (n-\nhat)\yt{t-1} -\nhat  \ystar)\rangle] \notag\\
&- \langle Q(\yt{t}-\yt{t-1}), n\yt{t} - (n-\nhat)\yt{t-1} - \nhat \ystar \rangle \notag \\
&- \EE\left[\frac{nv}{2}\|\xt{t} - \xstar\|^2  + \frac{n}{2}\|\xt{t} - \xstar\|_{H}^2 \right] \notag \\
&+ \frac{(n-\nhat)v}{2}\|\xt{t-1} - \xstar\|^2 + \frac{n }{2}\|\xt{t-1} - \xstar\|_{H}^2  \notag\\
&+\EE\left[-\rho \langle Z\xt{t-1} + B\yt{t}, Z(n\xt{t}-(n-\nhat)\xt{t-1}-\nhat \xstar)\rangle\right] \notag\\
& - \frac{n}{2}\EE\left[\|\xt{t} - \xt{t-1}\|^2_{H}\right]  -\frac{\nhat h'}{2}\|B^\top \lambdastar - \nabla \phi(\yt{t}) \|^2.
\label{eq:TheBasicFinequality}
\end{align}

\noindent {\it Step 2 {\rm (}Rearranging cross terms between $(\xt{t},\yt{t},\wt{t})$ and $(\xt{t-1},\yt{t-1},\wt{t-1})${\rm )}}:

Now, we define $\xhatt{t} := n\xt{t} - (n-\nhat)\xt{t-1}$ and $\yhatt{t} := n\yt{t} - (n-\nhat)\yt{t-1}$.
Then by the update rule of $\lambdat{t}$, we have that $\lambdat{t} = \lambdat{t-1} - \gamma \rho (Z\xhatt{t}+B\yhatt{t})$.
We evaluate the term 
$\EE[\langle \lambdat{t-1} - \lambdastar,Z(\xhatt{t}  - \nhat \xstar)\rangle] + \EE[\langle \lambdatilde{t} - \lambdastar,B(\yhatt{t} - \nhat \ystar)\rangle]$: 
\begin{align*}
&\langle \lambdat{t-1} - \lambdastar,Z(\xhatt{t}  - \nhat\xstar)\rangle + \langle \lambdatilde{t} - \lambdastar,B(\yhatt{t} - \nhat\ystar)\rangle \\
=
&\langle \lambdat{t-1} - \lambdastar,Z(\xhatt{t}  - \nhat\xstar)\rangle + 
\langle \lambdat{t-1} - \rho(Z\xt{t-1}+B\yt{t}) - \lambdastar,B(\yhatt{t} - \nhat\ystar) \rangle\\
=
&\langle \lambdat{t} +\gamma \rho (Z\xhatt{t}+B\yhatt{t}) - \lambdastar,Z(\xhatt{t}  - \nhat\xstar) \rangle \\ 
&+ 
\langle \lambdat{t} +\gamma \rho (Z\xhatt{t}+B\yhatt{t}) - \rho(Z\xt{t-1}+B\yt{t}) - \lambdastar,B(\yhatt{t} - \nhat\ystar)\rangle \\
=
&-\frac{1}{\gamma \rho} \langle \lambdat{t} - \lambdastar, \lambdat{t}-\lambdat{t-1}\rangle
\\ 
&+ \gamma \rho \|Z\xhatt{t}+B\yhatt{t}\|^2 - \rho \langle Z\xt{t-1}+B\yt{t}, B(\yhatt{t} - \nhat\ystar)\rangle \\
=
&-\frac{1}{2\gamma \rho} \left( \|\lambdat{t} - \lambdastar\|^2 + \|\lambdat{t}-\lambdat{t-1}\|^2 - \|\lambdat{t-1}-\lambdastar\|^2 \right)
\\ 
&+ \gamma \rho \|Z\xhatt{t}+B\yhatt{t}\|^2 - \rho \langle Z\xt{t-1}+B\yt{t}, B(\yhatt{t} - \nhat\ystar)\rangle \\
=
&\frac{1}{2\gamma \rho} \left( - \|\lambdat{t} - \lambdastar\|^2  + \|\lambdat{t-1}-\lambdastar\|^2 \right)
+ \frac{\gamma \rho}{2} \|Z\xhatt{t}+B\yhatt{t}\|^2 \\
&- \rho \langle Z\xt{t-1}+B\yt{t}, B(\yhatt{t} - \nhat\ystar)\rangle.
\end{align*}
Therefore, 
\begin{align*}
&\langle \lambdat{t-1} - \lambdastar,Z(\xhatt{t}  - \nhat\xstar) + \langle \lambdatilde{t} - \lambdastar,B(\yhatt{t} - \nhat\ystar)\rangle \\
&- \rho \langle Z\xt{t-1} + B\yt{t}, Z(n\xt{t}-(n-\nhat)\xt{t-1}-\nhat\xstar)\rangle \\
=&
\frac{1}{2\gamma \rho} \left( - \|\lambdat{t} - \lambdastar\|^2  + \|\lambdat{t-1}-\lambdastar\|^2 \right)
+ \frac{\gamma \rho}{2} \|Z\xhatt{t}+B\yhatt{t}\|^2 \\
&- \rho \langle Z\xt{t-1}+B\yt{t}, Z\xhatt{t} + B\yhatt{t} \rangle \\
=&
\frac{1}{2\gamma \rho} \left( - \|\lambdat{t} - \lambdastar\|^2  + \|\lambdat{t-1}-\lambdastar\|^2 \right) \\
&+ \frac{\gamma \rho}{2} n^2 \|Z\xt{t}+B\yt{t}\|^2 + \frac{\gamma \rho}{2} (n-\nhat)^2 \|Z\xt{t-1}+B\yt{t-1}\|^2 \\
&- \gamma \rho n(n-\nhat) \langle Z\xt{t}+B\yt{t},Z\xt{t-1}+B\yt{t-1}\rangle  \\
&- \rho \langle Z\xt{t-1}+B\yt{t}, Z(n\xt{t}-(n-\nhat)\xt{t-1}) + B(n\yt{t}-(n-\nhat)\yt{t-1}) \rangle.
\end{align*}
Next, we expand the non-squared term:
\begin{align}
&- \gamma \rho n(n-\nhat) \langle Z\xt{t}+B\yt{t},Z\xt{t-1}+B\yt{t-1}\rangle  \notag \\
&- \rho \langle Z\xt{t-1}+B\yt{t}, Z(n\xt{t}-(n-\nhat)\xt{t-1}) + B(n\yt{t}-(n-\nhat)\yt{t-1}) \rangle \notag \\
=
&- \gamma \rho n(n-\nhat) \langle Z\xt{t}-Z\xstar,Z\xt{t-1}-Z\xstar\rangle \notag \\
&- \gamma \rho n(n-\nhat) \langle B\yt{t}-B\ystar,B\yt{t-1}-B\ystar\rangle \notag \\
&- \gamma \rho n(n-\nhat) \langle Z\xt{t}-Z\xstar,B\yt{t-1}-B\ystar\rangle \notag \\
&- \gamma \rho n(n-\nhat) \langle B\yt{t}-B\ystar,Z\xt{t-1}-Z\xstar\rangle \notag \\
& - n \rho \langle Z\xt{t-1} - Z\xstar, Z(\xt{t}-\xstar)\rangle 
+(n-\nhat)  \rho \| Z\xt{t-1} - Z\xstar \|^2 \notag \\ 
& + (n-\nhat) \rho \langle B\yt{t} - B\ystar,  B(\yt{t-1}-\ystar) \rangle - n \rho \| B\yt{t} - B\ystar \|^2 \notag \\
&- \rho \langle Z\xt{t-1} - Z\xstar,  B(n\yt{t}-(n-\nhat)\yt{t-1} -\nhat\ystar) \rangle \notag \\
&- \rho \langle B(\yt{t}-\ystar), Z(n\xt{t}-(n-\nhat)\xt{t-1}-\nhat\xstar) \rangle \notag \\
=
&- (\gamma \rho n(n-\nhat)+n\rho) \langle Z\xt{t}-Z\xstar,Z\xt{t-1}-Z\xstar\rangle  \notag \\
&- (\gamma \rho n(n-\nhat)-(n-\nhat)\rho) \langle B\yt{t}-B\ystar,B\yt{t-1}-B\ystar\rangle  \notag \\
&- \gamma \rho n(n-\nhat) \langle Z\xt{t}-Z\xstar,B\yt{t-1}-B\ystar\rangle \notag \\
&- (\gamma \rho n(n-\nhat)+n\rho-(n-\nhat)\rho) \langle B\yt{t}-B\ystar,Z\xt{t-1}-Z\xstar\rangle  \notag \\
&
+(n-\nhat)  \rho \| Z\xt{t-1} - Z\xstar \|^2 - n \rho \| B\yt{t} - B\ystar \|^2 \notag \\
&- \rho (n-\nhat) \langle Z\xt{t-1} - Z\xstar,  B(\ystar-\yt{t-1}) \rangle \notag \\
&- \rho n \langle B(\yt{t}-\ystar), Z(\xt{t}-\xstar) \rangle.
\label{eq:Nonsquare_ADMMexpand}
\end{align}
Using the relation
\begin{align*}
&\langle Z\xt{t}-Z\xstar,B\yt{t-1}-B\ystar\rangle
= \langle Z(\xt{t}-\xstar),B(\yt{t}-\ystar)\rangle + \langle Z(\xt{t}-\xstar),B(\yt{t-1}-\yt{t})\rangle, \\
&\langle B\yt{t}-B\ystar,Z\xt{t-1}-Z\xstar\rangle
= \langle B(\yt{t}-\yt{t-1}),Z(\xt{t-1}-\xstar)\rangle + \langle B(\yt{t-1}-\ystar),Z(\xt{t-1}-\xstar)\rangle, 
\end{align*}
the RHS of \Eqref{eq:Nonsquare_ADMMexpand} is equivalent to  
\begin{align*}
&- (\gamma \rho n(n-\nhat)+n\rho) \langle Z\xt{t}-Z\xstar,Z\xt{t-1}-Z\xstar\rangle  \\
&- (\gamma \rho n(n-\nhat)-(n-\nhat)\rho) \langle B\yt{t}-B\ystar,B\yt{t-1}-B\ystar\rangle  \\
&
+(n-\nhat)  \rho \| Z\xt{t-1} - Z\xstar \|^2 - n \rho \| B\yt{t} - B\ystar \|^2  \\
&+ \{-(\gamma \rho n(n-\nhat)+\rho \nhat)  + \rho (n-\nhat)\} \langle Z\xt{t-1} - Z\xstar,  B(\yt{t-1}-\ystar) \rangle \\
&- (\gamma \rho n(n-\nhat) + \rho n) \langle B(\yt{t}-\ystar), Z(\xt{t}-\xstar) \rangle \\
&- \gamma \rho n(n-\nhat) \langle Z(\xt{t}-\xstar),B(\yt{t-1}-\yt{t})\rangle  \\
  & - (\gamma \rho n(n-\nhat)+\rho \nhat) \langle B\yt{t}-B\yt{t-1},Z\xt{t-1}-Z\xstar\rangle.
\end{align*}
The last two terms are transformed to 
\begin{align*}
&- \gamma \rho n(n-\nhat) \langle Z(\xt{t}-\xstar),B(\yt{t-1}-\yt{t})\rangle  \\
  & - (\gamma \rho n(n-\nhat)+\rho\nhat) \langle B\yt{t}-B\yt{t-1},Z\xt{t-1}-Z\xstar\rangle \\
=& \gamma \rho n(n-\nhat) \langle Z(\xt{t}-\xt{t-1}),B(\yt{t}-\yt{t-1})\rangle  \\
&- \rho \nhat\langle B\yt{t}-B\ystar,Z\xt{t-1}-Z\xstar\rangle
+ \rho \nhat\langle B\yt{t-1}-B\ystar,Z\xt{t-1}-Z\xstar\rangle.
\end{align*}
Thus, the RHS of \Eqref{eq:Nonsquare_ADMMexpand} is further transformed to 
\begin{align*}
&- (\gamma \rho n(n-\nhat)+n\rho) \langle Z\xt{t}-Z\xstar,Z\xt{t-1}-Z\xstar\rangle  \\
&- (\gamma \rho n(n-\nhat)-(n-\nhat)\rho) \langle B\yt{t}-B\ystar,B\yt{t-1}-B\ystar\rangle  \\
&+(n-\nhat)  \rho \| Z\xt{t-1} - Z\xstar \|^2 - n \rho \| B\yt{t} - B\ystar \|^2  \\
&+ \{-\gamma \rho n(n-\nhat)  + \rho(n-\nhat)\} \langle Z\xt{t-1} - Z\xstar,  B(\yt{t-1}-\ystar) \rangle \\
&- (\gamma \rho n(n-\nhat) + \rho n) \langle B(\yt{t}-\ystar), Z(\xt{t}-\xstar) \rangle \\
&+ \gamma \rho n(n-\nhat) \langle Z(\xt{t}-\xt{t-1}),B(\yt{t}-\yt{t-1})\rangle  \\
&- \rho \nhat\langle B\yt{t}-B\ystar,Z\xt{t-1}-Z\xstar\rangle.
\end{align*}
By Lemma \ref{lemm:abHcb} and $Z\xstar = - B\ystar$,
this is equivalent to 
\begin{align}
&- \frac{1}{2}(\gamma \rho n(n-\nhat)+n\rho)\{  \|Z\xt{t}-Z\xstar\|^2 + \|Z\xt{t-1}-Z\xstar\|^2 - \|Z\xt{t}-Z\xt{t-1}\|^2 \} \notag \\
&- \frac{1}{2}(\gamma \rho n(n-\nhat)-(n-\nhat)\rho)\{ \| B\yt{t}-B\ystar\|^2 + \|B\yt{t-1}-B\ystar\|^2 -  \|B\yt{t}-B\yt{t-1}\|^2 \} \notag\\
&
+(n-\nhat)  \rho \| Z\xt{t-1} - Z\xstar \|^2 - n \rho \| B\yt{t} - B\ystar \|^2 \notag \\
&- \frac{1}{2}\{-\gamma \rho n(n-\nhat)  + \rho(n-\nhat)\} (\| Z\xt{t-1} - Z\xstar\|^2 + \|B(\yt{t-1}-\ystar)\|^2 - \|Z\xt{t-1} + B\yt{t-1}\|^2)\notag \\
&+ \frac{1}{2}(\gamma \rho n(n-\nhat) + \rho n) (\| Z\xt{t} - Z\xstar\|^2 + \|B(\yt{t}-\ystar)\|^2 - \|Z\xt{t} + B\yt{t}\|^2) \notag\\
&+ \gamma \rho n(n-\nhat) \langle Z(\xt{t}-\xt{t-1}),B(\yt{t}-\yt{t-1})\rangle  \notag\\
&- \rho \nhat\langle B\yt{t}-B\ystar,Z\xt{t-1}-Z\xstar\rangle \notag\\
=
&- \frac{\rho\nhat}{2} \|Z\xt{t-1}-Z\xstar\|^2 + \frac{1}{2}(\gamma \rho n(n-\nhat)+n\rho) \|Z\xt{t}-Z\xt{t-1}\|^2  \notag\\
&- \frac{\rho\nhat}{2} \|B\yt{t}-B\ystar\|^2
+ \frac{1}{2}(\gamma \rho n(n-\nhat)-(n-\nhat)\rho) \|B\yt{t}-B\yt{t-1}\|^2 \notag\\
&- \frac{1}{2}\{\gamma \rho n(n-\nhat)  - \rho(n-\nhat)\} \|Z\xt{t-1} + B\yt{t-1}\|^2 \notag\\
&- \frac{1}{2}\{\gamma \rho n(n-\nhat) + \rho n\} \|Z\xt{t} + B\yt{t}\|^2 \notag\\
&+ \gamma \rho n(n-\nhat) \langle Z(\xt{t}-\xt{t-1}),B(\yt{t}-\yt{t-1})\rangle  \notag\\
&- \nhat \rho \langle B\yt{t}-B\ystar,Z\xt{t-1}-Z\xstar\rangle \notag\\
=
&- \frac{\nhat \rho}{2} \|Z\xt{t-1}+ B\yt{t}\|^2 \notag\\
&+ \frac{1}{2}(\gamma \rho n(n-\nhat)+n\rho) \|Z\xt{t}-Z\xt{t-1}\|^2 \notag \\
&+ \frac{1}{2}(\gamma \rho n(n-\nhat)-(n-\nhat)\rho) \|B\yt{t}-B\yt{t-1}\|^2 \notag\\
&- \frac{1}{2}\{\gamma \rho n(n-\nhat)  - \rho(n-\nhat)\} \|Z\xt{t-1} + B\yt{t-1}\|^2 \notag\\
&- \frac{1}{2}\{\gamma \rho n(n-\nhat) + \rho n\} \|Z\xt{t} + B\yt{t}\|^2 \notag\\
&+ \gamma \rho n(n-\nhat) \langle Z(\xt{t}-\xt{t-1}),B(\yt{t}-\yt{t-1})\rangle.
\label{eq:Nonsquare_ADMMexpand_tmp1}
\end{align}
Since 
\begin{align*}
& \gamma \rho n(n-\nhat) \langle Z(\xt{t}-\xt{t-1}),B(\yt{t}-\yt{t-1})\rangle  \\
\leq &
\frac{\gamma \rho n(n-\nhat) }{2}\{\|Z(\xt{t}-\xt{t-1})\|^2+\|B(\yt{t}-\yt{t-1})\|^2\}, 
\end{align*}
the RHS of \Eqref{eq:Nonsquare_ADMMexpand_tmp1} is bounded by 
\begin{align*}
&- \frac{\nhat\rho}{2} \|Z\xt{t-1}+ B\yt{t}\|^2 \\
&+ \frac{1}{2}(2\gamma \rho n(n-\nhat)+n\rho) \|Z\xt{t}-Z\xt{t-1}\|^2  \\
&+ \frac{1}{2}(2\gamma \rho n(n-\nhat)-(n-\nhat)\rho) \|B\yt{t}-B\yt{t-1}\|^2 \\
&- \frac{1}{2}\{\gamma \rho n(n-\nhat)  - \rho(n-\nhat)\} \|Z\xt{t-1} + B\yt{t-1}\|^2  
- \frac{1}{2}\{\gamma \rho n(n-\nhat) + \rho n\} \|Z\xt{t} + B\yt{t}\|^2.
\end{align*}
Combining this and \Eqref{eq:TheBasicFinequality}, and noticing $\|Z\xt{t}-Z\xt{t-1}\| = \|Z_I(\xt{t}_I-\xt{t-1}_I)\|=\|\xt{t}-\xt{t-1}\|_{\DiagI(Z^\top Z)}$, we obtain 
\begin{align*}
&n \EE[F(\xt{t},\yt{t}) - F(\xt{t-1},\yt{t-1})] \\
\leq
&
\nhat(F(\xstar,\ystar) - F(\xt{t-1},\yt{t-1})) \\
&+ \frac{1}{2\gamma \rho} \left( - \|\lambdat{t} - \lambdastar\|^2  + \|\lambdat{t-1}-\lambdastar\|^2 \right) \\
&- \frac{\nhat \rho}{2} \|Z\xt{t-1}+ B\yt{t}\|^2 \\
&+ \frac{1}{2}\{ \gamma \rho n^2 - \gamma \rho n(n-\nhat) - \rho n \}\|Z\xt{t}+B\yt{t}\|^2 \\
&+ \frac{1}{2} \{\gamma \rho(n-\nhat)^2 - \gamma \rho n(n-\nhat)  + \rho(n-\nhat)\} \|Z\xt{t-1}+B\yt{t-1}\|^2 \\
&- \EE\left[\frac{nv}{2}\|\xt{t} - \xstar\|^2 + \frac{n}{2}\|\xt{t} - \xstar\|_{H}^2 \right] \notag\\
&+ \frac{(n-\nhat)v}{2}\|\xt{t-1} - \xstar\|^2  + \frac{n}{2}\|\xt{t-1} - \xstar\|_{H}^2  \notag\\
& + \gamma \rho n(n-\nhat)\EE\left[\|\xt{t} - \xt{t-1}\|^2_{\DiagI(Z^\top Z)}\right] 
 - \frac{n}{2}\EE[\|\xt{t}-\xt{t-1}\|^2_{\DiagI(G)}] \\
& + (\gamma \rho n (n-\nhat) - \frac{(n-\nhat)\rho}{2}) \|B(\yt{t} - \yt{t-1})\|^2 \notag \\
&- \langle Q(\yt{t}-\yt{t-1}), n\yt{t} - (n-\nhat)\yt{t-1} - \nhat \ystar \rangle \\
& -\frac{\nhat h'}{2} \|B^\top \lambdastar - \nabla \phi(\yt{t}) \|^2.
\end{align*}
Since we have assumed $\DiagI(G) \succ 2\gamma \rho(n-\nhat)\DiagI(Z^\top Z)$,
it holds that 
$$
\gamma \rho n(n-\nhat)\EE\left[\|\xt{t} - \xt{t-1}\|^2_{\DiagI(Z^\top Z)}\right] 
 - \frac{n}{2}\EE[\|\xt{t}-\xt{t-1}\|^2_{\DiagI(G)}] \leq 0.
$$
Moreover, we have that
\begin{align*}
&- \langle Q(\yt{t}-\yt{t-1}), n\yt{t} - (n-\nhat)\yt{t-1} - \nhat\ystar \rangle \\
=
&-n\|\yt{t}-\yt{t-1}\|_Q^2 + \frac{1}{2}\{\|\yt{t}-\yt{t-1}\|^2_Q + \|\yt{t-1}-\ystar\|^2_Q - \|\yt{t}-\ystar\|^2_Q \} \\
=
&-\left(n-\frac{\nhat}{2}\right)\|\yt{t}-\yt{t-1}\|_Q^2 + \frac{\nhat}{2}\|\yt{t-1}-\ystar\|^2_Q - \frac{\nhat}{2}\|\yt{t}-\ystar\|^2_Q. 
\end{align*}
Finally, we achieve 
\begin{align}
&n \EE[F(\xt{t},\yt{t}) - F(\xt{t-1},\yt{t-1})] \notag \\
\leq
&
\nhat (F(\xstar,\ystar) - F(\xt{t-1},\yt{t-1})) \notag\\
&+ \frac{1}{2\gamma \rho} \left( - \|\lambdat{t} - \lambdastar\|^2  + \|\lambdat{t-1}-\lambdastar\|^2 \right) \notag\\
&- \frac{\rho n(1-\gamma)}{2}\|Z\xt{t}+B\yt{t}\|^2 + \frac{\rho (n-\nhat)(1+\gamma)}{2} \|Z\xt{t-1}+B\yt{t-1}\|^2 \notag\\
&- \EE\left[\frac{nv}{2}\|\xt{t} - \xstar\|^2  + \frac{n}{2}\|\xt{t} - \xstar\|_{H}^2\right] \notag\\
&+ \frac{(n-\nhat)v}{2}\|\xt{t-1} - \xstar\|^2  + \frac{n \rho}{2}\|\xt{t-1} - \xstar\|_{H}^2 \notag\\
& + \gamma \rho n(n-\nhat)\EE\left[\|\xt{t} - \xt{t-1}\|^2_{\DiagI(Z^\top Z)}\right] 
 - \frac{n}{2}\EE[\|\xt{t}-\xt{t-1}\|^2_{\DiagI(G)}] \notag\\
&- \frac{\nhat \rho}{2} \|Z\xt{t-1}+ B\yt{t}\|^2 \notag\\
& + (\gamma \rho n (n-\nhat) - \frac{(n-\nhat)\rho}{2}) \|B(\yt{t} - \yt{t-1})\|^2 \notag \\
& -\left(n-\frac{\nhat}{2}\right)\|\yt{t}-\yt{t-1}\|_Q^2 + \frac{\nhat}{2}\|\yt{t-1}-\ystar\|^2_Q - \frac{\nhat}{2}\|\yt{t}-\ystar\|^2_Q \notag\\
& -\frac{\nhat h'}{2}\|B^\top \lambdastar - \nabla \phi(\yt{t}) \|^2.
\label{eq:FinalBoundBeforeAssertion}
\end{align}
Note that \Eqref{eq:FinalBoundBeforeAssertion} holds for arbitrary $\ystar \in \calY^*$.

\noindent {\it Step 3: {\rm(}Deriving the assertion{\rm)}}

(i)
Now, since $\nabla \phi(\yt{t}) = B^\top \lambdat{t-1} - \rho (Z \xt{t-1} + B\yt{t}) - Q(\yt{t}-\yt{t-1})$, it holds that
$$
\|B^\top \lambdastar - \nabla \phi(\yt{t}) \|^2
=\| B^\top( \lambdastar - \lambdat{t-1}) -  \rho (Z \xt{t-1} + B\yt{t}) - Q(\yt{t}-\yt{t-1})\|^2.
$$
Since $B^\top$ is injection, this gives that
\begin{align*}
& - \frac{h'}{2}\|B^\top \lambdastar - \nabla \phi(\yt{t}) \|^2 \\
\leq & - h' \sigma_{\min}(B B^\top) \|  \lambdastar - \lambdat{t-1}\|^2 + 2 h' \rho^2 \|Z \xt{t-1} + B\yt{t}\|^2 + 2 h' \|Q(\yt{t}-\yt{t-1})\|^2 \\
\leq & - h' \sigma_{\min}(B B^\top) \|  \lambdastar - \lambdat{t-1}\|^2 + 2 h' \rho^2 \|Z \xt{t-1} + B\yt{t}\|^2 + 2 h' \sigma_{\max}(Q) \|\yt{t}-\yt{t-1}\|_Q^2.
\end{align*}
Now, dividing both sides by $\max\{1,4h'\rho,4h'\sigma_{\max(Q)}\}~ (\geq 1)$, we have 
\begin{align}
& - \frac{h'}{2}\|B^\top \lambdastar - \nabla \phi(\yt{t}) \|^2 \notag\\
\leq & - \frac{h'\sigma_{\min}(B B^\top)}{\max\{1,4h'\rho,4h'\sigma_{\max(Q)}\}}  \|  \lambdastar - \lambdat{t-1}\|^2 + \frac{\rho}{2} \|Z \xt{t-1} +  B\yt{t}\|^2 + \frac{1}{2} \|\yt{t}-\yt{t-1}\|_Q^2.
\label{eq:BtPhiDistinctBound}
\end{align}

(ii)
Next, it holds that, for some $\ystarhat \in \calY^*$, 
\begin{align}
\frac{1}{2}\left(F(\xstar,\ystar) - F(\xt{t-1},\yt{t-1})\right)
& \leq 
 - \frac{v_\phi'}{4}\|P_{\Ker(B)}(\yt{t-1}-\ystarhat)\|^2. 
\label{eq:AxPker}
\end{align}
On the other hand, for arbitrary $a > 0$, it follows that
\begin{align*}
&- \frac{\rho}{8} \|Z \xt{t-1} + B\yt{t-1}\|^2  \\
\leq & - \frac{1}{8}(1-a) \|Z (\xt{t-1} - \xstar) \|^2 - \frac{1}{8}(1-a^{-1}) \|B(\yt{t-1} - \ystarhat)\|^2.
\end{align*}
Thus, setting $a = 1+\frac{2v}{\rho\sigma_{\max}(Z^\top Z)}$, 
we have that
\begin{align}
&- \frac{\rho}{8} \|Z \xt{t-1} + B\yt{t-1}\|^2  \notag \\
\leq &  \frac{\rho}{8} \frac{2v}{\rho\sigma_{\max}(Z^\top Z)} \sigma_{\max}(Z^\top Z) \| \xt{t-1} - \xstar \|^2 \notag \\
&- \frac{\rho}{8}\frac{2v\rho}{\rho\sigma_{\max}(Z^\top Z) + 4v} \sigma_{\min}(BB^\top) \|P_{\Ker(B)}^\perp(\yt{t-1} - \ystarhat)\|^2 \notag \\
= &
 \frac{v}{4} \| \xt{t-1} - \xstar \|^2 - \frac{v\rho\sigma_{\min}(BB^\top)}{4(\rho \sigma_{\max}(Z^\top Z) + 4v)} \|P_{\Ker(B)}^\perp(\yt{t-1} - \ystarhat)\|^2.
\label{eq:AxPkerPerp}
\end{align}
Combining Eqs. \eqref{eq:AxPker}, \eqref{eq:AxPkerPerp}, we have that 
\begin{align}
&\frac{\nhat}{2n}\left(F(\xstar,\ystar) - F(\xt{t-1},\yt{t-1})\right) - \frac{\nhat \rho}{8n} \|Z \xt{t-1} + B\yt{t-1}\|^2 \notag \\
\leq
&
\frac{\nhat v}{4n} \| \xt{t-1} - \xstar \|^2 
- \frac{\nhat}{n^2}\min\left \{ nv_\phi', 
\frac{n\rho v\sigma_{\min}(BB^\top)}{\rho \sigma_{\max}(Z^\top Z) + 4v} \right\} 
\frac{\|\yt{t-1} - \ystarhat\|_Q^2}{4\sigma_{\max(Q)}} \notag \\
\leq
&
\frac{\nhat v}{4n} \| \xt{t-1} - \xstar \|^2 
- \frac{\nhat}{n^2}\min\left \{ nv_\phi', 
\frac{n\rho v\sigma_{\min}(BB^\top)}{\rho \sigma_{\max}(Z^\top Z) + 4v} \right\} 
\frac{\|\yt{t-1} - \calY^*\|_Q^2}{4\sigma_{\max(Q)}}.
\label{eq:FtStrongConvexityResidual}
\end{align}

(iii)
Therefore, if $\gamma = \frac{1}{4n}$, 
applying \Eqref{eq:BtPhiDistinctBound} and \Eqref{eq:FtStrongConvexityResidual} to \Eqref{eq:FinalBoundBeforeAssertion},
for  
\begin{align*}
\nu = \frac{\nhat}{n}\min &\Bigg\{ \frac{1}{4}\left(\frac{v}{v+\sigma_{\max}(H)} \right),
\frac{h'\rho\sigma_{\min}(B B^\top)}{2\max\{1,4h'\rho,4h'\sigma_{\max(Q)}\}}, \frac{n v_\phi'}{4\sigma_{\max(Q)}}, 
\frac{n v\sigma_{\min}(BB^\top)}{4\sigma_{\max(Q)} (\rho \sigma_{\max}(Z^\top Z) + 4v)}   \Bigg\},
\end{align*}
we have that 
\begin{align*}
& \EE\Big[F(\xt{t},\yt{t}) - F(\xstar,\ystar)
+ \frac{1}{2 n \gamma \rho}\|\lambdat{t} - \lambdastar\|^2 \\
&+\frac{\rho(1-\gamma)}{2}\|Z\xt{t}+B\yt{t}\|^2
+ \frac{1}{2}\|\xt{t} - \xstar\|^2_{vI_p + H} + \frac{1}{2n}\|\yt{t}-\ystar\|^2_Q \Big]
\\
\leq
&
\left(1-\nu\right)
\Bigg \{F(\xt{t-1},\yt{t-1}) - F(\xstar,\ystar)  +   
\frac{1}{2n\gamma \rho} \|\lambdat{t-1} - \lambdastar\|^2   \\
& + \frac{\rho(1-\gamma)}{2} \|Z\xt{t-1}+B\yt{t-1}\|^2  
+
\frac{1}{2}\|\xt{t-1} - \xstar\|^2_{vI_p + H}
+ \frac{1}{2n}\|\yt{t-1}-\ystar\|^2_Q
\Bigg\}.
\end{align*}
Setting $\mu := n \nu/\nhat$, this gives the assertion.

\end{proof}

\begin{Lemma}
\label{lemm:crosstermExpAxBy}
\begin{align*}
&\EE\left[ - \rho \langle  Z_{\backslash i} \xt{t}_{\backslash i} + B \yt{t}, Z_I (\xt{t}_I - \xstar_I) \rangle 
+ \frac{\rho}{2} \|Z_I \xstar_I \|^2 - \frac{\rho}{2} \|Z_I \xt{t}_I \|^2 \right] \\
\leq 
&\EE\left[-\frac{\rho}{n} \langle Z\xt{t-1} + B\yt{t}, Z(n\xt{t}-(n-\nhat)\xt{t-1}-\xstar)\rangle\right]  \\
&+\frac{\rho}{2n}\|\xt{t-1} - \xstar\|^2_{\DiagI(Z^\top Z)} - \frac{\rho}{2}\EE\left[\|\xt{t} - \xt{t-1}\|^2_{\DiagI(Z^\top Z)}\right].
\end{align*}

\end{Lemma}

\begin{proof}
\begin{align*}
&\rho\langle Z_{\backslash I}\xt{t-1}_{\backslash I}, Z_I(\xstar_I - \xt{t}_I) \rangle
+\rho\langle B\yt{t}, Z_I(\xstar_I - \xt{t}_I) \rangle
+ \frac{\rho}{2} \|Z_I \xstar_I \|^2 - \frac{\rho}{2} \|Z_I \xt{t}_I \|^2 \\
=
&
\rho\langle Z\xt{t-1}, Z_I(\xstar_I - \xt{t}_I) \rangle
+\rho\langle B\yt{t}, Z_I(\xstar_I - \xt{t}_I) \rangle
+ \frac{\rho}{2} \|Z_I \xstar_I \|^2 - \frac{\rho}{2} \|Z_I \xt{t}_I \|^2 \\
&- \rho\langle Z_I\xt{t-1}_{i}, Z_I(\xstar_I - \xt{t}_I) \rangle \\
=
&
\rho\langle Z\xt{t-1}, Z_I(\xstar_I -\xt{t-1}_I+\xt{t-1}_I- \xt{t}_I) \rangle
+\rho\langle B\yt{t}, Z_I(\xstar_I -\xt{t-1}_I+\xt{t-1}_I- \xt{t}_I) \rangle \\
&+ \frac{\rho}{2} \|Z_I \xstar_I \|^2 - \frac{\rho}{2} \|Z_I \xt{t}_I \|^2 - \rho\langle Z_I\xt{t-1}_I, Z_I(\xstar_I - \xt{t}_I) \rangle \\ 
=
&
\rho\langle Z\xt{t-1} + B\yt{t}, Z_I(\xstar_I- \xt{t-1}_I) \rangle \\
&+ \frac{\rho}{2} \|Z_I(\xt{t-1}_I-\xstar_I) \|^2 - \frac{\rho}{2} \|Z_I (\xt{t}_I - \xt{t-1}_I) \|^2   \\
&+\rho\langle Z\xt{t-1} + B\yt{t}, Z(\xt{t-1} -\xt{t}) \rangle.
\end{align*}
The expectation of the RHS is evaluated as 
\begin{align*}
&\frac{\rho}{n}\langle Z\xt{t-1} + B\yt{t}, Z(\xstar- \xt{t-1}) \rangle + \frac{\rho}{2n} \|\xt{t-1}-\xstar \|_{\DiagI(Z^\top Z)}^2 \\
&- \frac{\rho}{2} \EE[\|Z (\xt{t} - \xt{t-1}) \|^2] +\rho \EE[\langle Z\xt{t-1} + B\yt{t}, Z(\xt{t-1} -\xt{t}) \rangle] \\
=
&-\frac{\rho}{n} \EE[\langle Z\xt{t-1} + B\yt{t}, Z(n\xt{t}- (n-\nhat)\xt{t-1} - \xstar) \rangle]  \\
& + \frac{\rho}{2n} \|\xt{t-1}-\xstar \|_{\DiagI(Z^\top Z)}^2 - \frac{\rho}{2} \EE[\|Z (\xt{t} - \xt{t-1}) \|^2].
\end{align*}
This gives the assertion.
\end{proof}

\begin{Lemma}
\label{lemm:psidifbound}
For all $y \in \Real^d$ and $\ystar \in \calY^*$, we have  
\begin{align*}
\phi(y) - \phi(\ystar)
\leq 
\langle \nabla \phi(y), y - \ystar \rangle
-\frac{h'}{2} \|\nabla \phi(y) - \nabla \phi(\ystar) \|^2.
\end{align*}

\end{Lemma}
\begin{proof}
By assumption, for all $\ystar \in \calY^*$, we have that 
\begin{align*}
\phi(y) &= -\phi^*(\nabla \phi(y)) + \langle y, \nabla \phi(y) \rangle\\
&\leq - \phi^*(\nabla \phi(\ystar)) + \langle \ystar, \nabla \phi(\ystar) - \nabla \phi(y) \rangle
-\frac{h'}{2}\|\nabla \phi(\ystar) - \nabla \phi(y)\|^2
+ \langle y, \nabla \phi(y) \rangle \\
&= 
\langle \nabla \phi(\ystar),\ystar \rangle + \phi(\ystar) + \langle \ystar, \nabla \phi(\ystar) - \nabla \phi(y) \rangle
-\frac{h'}{2}\|\nabla \phi(\ystar) - \nabla \phi(y)\|^2
+ \langle y, \nabla \phi(y) \rangle \\
&=
\langle \nabla \phi(\ystar),\ystar \rangle + \phi(\ystar) + \langle \ystar, \nabla \phi(\ystar) - \nabla \phi(y) \rangle
-\frac{h'}{2}\|\nabla \phi(\ystar) - \nabla \phi(y)\|^2
+ \langle y, \nabla \phi(y) \rangle \\
&=
 \phi(\ystar) + \langle y - \ystar,\nabla \phi(y) \rangle
-\frac{h'}{2}\|\nabla \phi(\ystar) - \nabla \phi(y)\|^2.
\end{align*}

\end{proof}

\section{Auxiliary Lemmas}
\begin{Lemma}
\label{lemm:abHcb}
For all symmetric matrix $H$, we have 
\begin{align}
(a-b)^\top H (c-b) 
= \frac{1}{2} \|a -b\|_H^2 - \frac{1}{2}\|a-c\|_H^2 + \frac{1}{2}\|c-b\|_H^2.
\end{align}
\end{Lemma}
\begin{proof}
\begin{align*}
(a-b)^\top H (c-b) 
&= \left(a-\frac{c+b}{2} + \frac{c+b}{2} - b\right)^\top H (c-b) \\
&= \left(\frac{a-c}{2} + \frac{a-b}{2}\right)^\top H (c-b)
+  \left(\frac{c-b}{2}\right)^\top H (c-b) \\
&= \left(\frac{a-c}{2} + \frac{a-b}{2}\right)^\top H \{(a-b) - (a-c)\}
+  \left(\frac{c-b}{2}\right)^\top H (c-b) \\
&= \frac{1}{2}\|a-b\|_H^2 - \frac{1}{2}\|a-c\|_H^2 + \frac{1}{2}\|c-b\|_H^2.
\end{align*}
\end{proof}

\end{document}